\documentclass{article}
\usepackage{verbatim}
\usepackage{authblk}
\usepackage[utf8]{inputenc}
\usepackage{amssymb}
\usepackage{amsmath}
\usepackage{amsthm}
\usepackage{xspace}
\usepackage{xcolor}
\usepackage{algorithm}
\usepackage{algpseudocode}
\usepackage{graphicx}
\usepackage{mathtools}
\usepackage{url}

\newtheorem{lemma}{Lemma}
\newtheorem{definition}{Definition}
\newtheorem{theorem}{Theorem}
\newtheorem{corollary}{Corollary}
\newenvironment{proofof}[1]{\begin{proof}[Proof of~#1]}{\end{proof}}

\newcommand{\ooea}{(1+1)~EA\xspace}

\newcommand{\mupoea}{($\mu$+1)~EA\xspace}
\newcommand{\muoneea}{\mupoea}
\newcommand{\muporls}{($\mu$+1)~RLS\xspace}

\newcommand{\om}{\textsc{OneMax}\xspace}
\newcommand{\onemax}{\om}
\newcommand{\twomax}{\textsc{TwoMax}\xspace}
\newcommand{\sufsamp}{\textsc{RidgeWithBranches}\xspace}
\newcommand{\truncatedtwomax}{\textsc{TruncatedTwoMax$_k$}\xspace}

\newcommand{\R}{\ensuremath{\mathbb{R}}}
\newcommand{\N}{\ensuremath{\mathbb{N}}}
\newcommand{\Bin}{\ensuremath{\mathrm{Bin}}} 
\newcommand{\Bernoulli}{\ensuremath{\mathrm{Bernoulli}}} 

\newcommand{\ie}{i.\,e.\xspace}
\newcommand{\eg}{e.\,g.\xspace}

\newcommand{\wop}{w.\,o.\,p.\xspace}

\newcommand{\abs}[1]{\lvert #1\rvert}
\newcommand{\onenorm}[1]{\lvert #1\rvert_{1}}
\newcommand{\ones}[1]{\lvert #1\rvert_{1}}
\newcommand{\card}[1]{\abs{#1}}

\usepackage{tikz}
\usepackage{pgfplots}
\pgfplotsset{compat = 1.14}
\usepgfplotslibrary{fillbetween}
\usepgfplotslibrary{statistics}
\usepackage{pgfplotstable} 
\usetikzlibrary{decorations.markings}
\usetikzlibrary{calc}

\newcommand{\plotdatapath}[1]{figures/data/#1}

\newcommand{\plotwitherr}[6]{
\addplot[forget plot,name path=hierr,mark=none,draw=none] table[x=#3,y=EH/#4] {\plotdatapath{err-#6}};%
\addplot[forget plot,name path=loerr,mark=none,draw=none] table[x=#3,y=EL/#4] {\plotdatapath{err-#6}};%
\addplot[forget plot,opacity=0.1,#1,on layer=pre main] fill between[of=hierr and loerr];%
\addplot[#1,#2] table[x=#3,y=#4] {\plotdatapath{#6}};
}
\newcommand{\plotwitherx}[6]{\plotwitherr{#1}{#2,dotted,forget plot}{#3}{#4}{#5}{#6}}

\newcommand{\boxplotheight}{0.32\linewidth}

\newcommand{\peter}[1]{\textcolor{black}{ #1}}

\DeclareMathOperator{\Prob}{Pr}

\DeclareMathOperator{\polylog}{polylog}

\DeclareMathOperator{\PO}{PO}
\DeclareMathOperator{\LSO}{LSO}
\newcommand{\sufpar}{\textsc{TwoGradients}}
\newcommand{\prob}[1]{\Pr\left(#1\right)}

\usepackage{booktabs} 

\begin{document}

\title{
On Steady-State Evolutionary Algorithms and Selective Pressure: Why Inverse Rank-Based Allocation of Reproductive Trials is Best}

\author[1]{Dogan~Corus\thanks{d.corus@sheffield.ac.uk}}
\author[1]{Andrei~Lissovoi\thanks{a.lissovoi@sheffield.ac.uk}}
\author[1]{Pietro~S.~Oliveto\thanks{p.oliveto@sheffield.ac.uk}}
\author[2]{Carsten~Witt\thanks{cawi@imm.dtu.dk}}

\affil[1]{The University of Sheffield}
\affil[2]{Technical University of Denmark}

\renewcommand\Authands{ and }
				
				%
%
%
%
%

\begin{abstract}

We analyse the impact of the selective pressure for the global optimisation capabilities of steady-state EAs. For the standard bimodal benchmark function \twomax we rigorously prove that using uniform parent selection leads to exponential runtimes with high probability to locate both optima for the standard ($\mu$+1)~EA and ($\mu$+1)~RLS with any polynomial population sizes. 
On the other hand, we prove that selecting the worst individual as parent leads to efficient global optimisation with overwhelming probability for reasonable population sizes. Since always selecting the worst individual may have detrimental effects for escaping from local optima, we consider the performance of stochastic parent selection operators with low selective pressure for a function class called  \textsc{TruncatedTwoMax} where one slope is shorter than the other. An experimental analysis shows that the EAs equipped with inverse tournament selection, where the loser is selected for reproduction and small tournament sizes, globally optimise \twomax efficiently and effectively escape from local optima of  \textsc{TruncatedTwoMax} with high probability. Thus they identify both optima efficiently while uniform (or stronger) selection fails in theory and in practice. We then show the power of inverse selection on function classes from the literature where populations are essential by providing rigorous proofs or experimental evidence that it outperforms uniform selection equipped with or without a restart strategy. We conclude the paper by confirming our theoretical insights with an empirical analysis of the different selective pressures on standard benchmarks of the classical MaxSat and Multidimensional Knapsack Problems.
\end{abstract}


\maketitle

\section{Introduction}
In natural evolution, the better individuals adapt to the environment, 
the more likely it is that they are selected for reproduction -- that they generate offspring~\cite{Darwin1859}.
Indeed, 
canonical {\it generational} genetic algorithms (GAs) rely on this concept to efficiently evolve populations of candidate solutions~\cite{Goldberg1989,Holland1992}.
Generational evolutionary algorithms (EAs) model natural evolution by using the current population to create a completely new population of the same size in each generation (i.e., all individuals die by the end of their generation). The more individuals of the current population are fit {\color{black} (i.e., have high objective function value)}, the more likely they are to generate offspring for the next generation. These offspring will be part of the next population independently of their fitness but the probability they will reproduce, in turn, depends on how fit they are compared to the rest of the population.
All the commonly used selection operators in generational EAs and GAs (such as {\color{black}e.g.}, fitness proportional, tournament, ranking and comma selection) rely on this described property of natural evolution to increase the average fitness of the population over time.
Roughly speaking, the probability that fit individuals are selected for reproduction and produce a surviving offspring in generational EAs and GAs is commonly referred to as {\it selective pressure}. The higher the selective pressure, the faster diversity is lost due to the best individuals taking over the population (i.e., {\it takeover time}).
On the other hand, it is well understood that sufficiently high selective pressure is required for the population of a generational EA to evolve effectively~\cite{LehreEtAl2018}. In particular, if the selective pressure is not high enough, then generational EAs will require exponential time to optimise any function with unique optimum~\cite{Lehre2010}.

More recently introduced evolutionary algorithms, typically referred to as {\it steady-state GAs} \cite{Whitley1989,Syswerda1989} (or {\it steady-state EAs} according to whether they use crossover or not~\cite{WittFamily}), rely on {\it elitist} selection (often referred to also as {\it cut, truncation} or {\it plus} selection)  between consecutive generations. 
Differently to canonical ones, steady-state algorithms (EAs or GAs) evolve overlapping populations. They create an offspring population of smaller size (i.e.\ $\lambda$) compared to the parent population (of size $\mu>\lambda$) and the worst individuals from the union of both populations (i.e. of total size $\mu$+$\lambda$) are removed to obtain a new population of the same size as that of the parent population for the next generation.
This considerable change in the evolutionary process means that the best genotypes do not disappear from the population in the next generation
even if they do not reproduce. Such a difference implies that, if elitism is used, it is not {\it crucial} that better individuals have higher probabilities of reproducing for a population to evolve effectively.
In particular, from the point of view of the  {\it survival of the fittest}~\cite{Darwin1868,Spencer}, elitism makes the choice of how parents are selected for reproduction irrelevant (i.e.,  the probability that the best $\mu-\lambda$ individuals will be in the next generation is 1). 
From this point of view, the choice of the parent selection mechanism becomes arbitrary. 
While when steady-state evolutionary algorithms were originally introduced in the genetic algorithm community individuals were ranked and selected with a bias towards the top~\cite{Whitley1989,Syswerda1989}, indeed nowadays
parents are commonly selected uniformly at random for reproduction in steady-state EAs and GAs\footnote{This is probably due to the fact that the German {\it evolution strategies}  community, where continuous optimisation though was the main focus, used uniform parent selection in ($\mu+\lambda$)~ESs~\cite{BackHoffmeisterSchwefel1991}. 
While for the mentioned reasons this clearly makes sense, very successful steady-state GAs may still apply some parent selective pressure such as the binary tournament selection used in Chu and Beasley's GA. ~\cite{Chu1998}.}~\cite{WittFamily, FriedrichOSWECJ09, DangEtAlTEVC2018, CorusOlivetoTEVC2018, Sutton2018, Lengler2018}.

From a different perspective, the parent selection operator may considerably impact the overall performance of the algorithm (i.e., the expected time for the global optimum to be found).
In particular, we argue that the selective pressure commonly used in steady-state EAs is too high 
because, the higher the probability of selecting better individuals, the quicker the population is taken over by the best one, hence the faster {\it premature convergence}~\cite{Goldberg1989,SudholtBookChapter2019} occurs.
Such a problem was already argued shortly after the introduction of steady-state GAs to the community~\cite{DebGoldberg1991}.   
However, the only solutions considered to avoid premature convergence with overlapping populations were either to increase the population
size or to use multiple populations (niching).

In this paper we analyse the impact of parent selection on the performance of steady-state EAs for global optimisation.
To this end, we consider the performance of the well studied ($\mu$+1)~EA for the standard bimodal benchmark  function \twomax. The function has been used in the theory of  evolutionary computation to evaluate the exploration capabilities of EAs in discrete search spaces because, since the optima are located very far away from each other, it is very challenging to evolve a population that contains both of them~\cite{FriedrichOSWECJ09, SudholtBookChapter2019, OlivetoSudholtZarges2018, CovantesOsunaSudholt2017}.
In particular, \twomax has often been used to evaluate the performance of diversity enhancing mechanisms for the ($\mu$+1)~EA because, without such mechanisms, the algorithm has been proven to be unable to identify both optima of the function in polynomial time with high probability if the population size is at most sublinear (i.e. $\mu=o(n/\log n)$)~\cite{FriedrichOSWECJ09}.

However, no proof is available that the algorithm would not be efficient using larger population sizes. Our first result is the proof that it is not. In particular, we show that the ($\mu$+1)~EA and the ($\mu$+1)~RLS with standard uniform parent selection cannot identify both optima of \twomax in polynomial time with 1 - $o(1)$ probability for any population of polynomial size.
The previous analysis from the literature assumed pessimistically that the population has nearly optimised both branches when the takeover occurs.
Our analysis reveals that the takeover occurs much earlier: not long after initialisation. Hence, the negative effects of premature convergence are emphasised\textcolor{black}{: while \twomax displays local and global optima at opposite extremities of the search space, since the takeover occurs very early it may negatively affect the algorithm on much larger classes of multimodal functions.}

On the other hand, we also prove that it is not necessary to use diversity enforcing mechanisms for the two algorithms to identify both optima of \twomax: decreasing the selective pressure suffices!
As a proof of concept we decrease it to a minimum by always selecting the worst individual of the population as parent (breaking ties uniformly at random).
We prove that such a decrease leads the algorithms to the efficient global optimisation of \twomax for reasonable population sizes
i.e., the expected runtime is $O(\mu n \log n)$ with 
$\mu=\Omega(n^{1+\epsilon})$ for the ($\mu$+1)~RLS and 
$\mu=\Omega(n^{1+2\epsilon})$ for the ($\mu$+1)~EA.

However, while we are arguing for low selective pressures, in general we do not recommend selecting the worst individual for reproduction as it may have serious detrimental effects for the global optimisation capabilities of steady-state EAs.
For example, if the worst individual of the population is stuck on a local optimum, then it will be selected for reproduction until it manages to escape.
For some local optima this may be prohibitive, and even require super-polynomial time. To tackle this problem  we
consider the usage of stochastic parent selection operators that do not deterministically select the worst individual of the population. 
To this end, we consider tournament selection with (small) tournament sizes where the {\it loser} of the tournament is chosen for reproduction\peter{\footnote{We use inverse tournament for simplicity of implementation. The same results can be achieved with Inverse Rank Selection.}}.
Thus, we increase the selective pressure while still keeping it lower than that of uniform selection.
To evaluate the performance of the algorithm at escaping from local optima, we modify \twomax such that one slope is shorter than the other and call the resulting function class \textsc{TruncatedTwoMax$_k$} where the parameter $k$ defines the height of the local optima.

A corollary of our first result shows that the \mupoea with standard uniform selection will only identify one of the peaks with high probability and this will be the lower one with probability at least $1/2 - o(1)$.
An experimental analysis confirms that {\it Inverse Elitist} selection is very effective for the global optimisation of \twomax and that for \textsc{TruncatedTwoMax} with a probability of approximately 1/2 it gets stuck on the lower branch unless an unlikely takeover happens before. On the other hand, the experiments reveal that {\it Inverse Tournament Selection} is able to escape from the local optima, hence is efficient for both \twomax and \textsc{TruncatedTwoMax} with approximately linear population sizes.

Although the \twomax benchmark function is commonly used to evaluate the exploration capabilities of population based EAs, a single trajectory algorithm with a restart strategy suffices to optimise the function very efficiently.
Hence, we also evaluate the performance of the {\it Inverse Elitist} and {\it Inverse Tournament Selection} on benchmark functions that have previously been used in the literature to highlight problem characteristics where populations are essential.
In particular, problem classes where single trajectory EAs fail with overwhelming probability (w.o.p.), 
hence restart strategies are also ineffective. 
To this end, we pick the multimodal  function \textsc{SufsAmp} (which we generalise and prefer to call \sufsamp) 
previously used in the literature to highlight fitness landscape characteristics where the use of offspring populations is essential
~\cite{JansenSufSamp}
and a bimodal function class, that we call \sufpar, which has been used in the literature to 
highlight circumstances where parent populations are essential~\cite{WittFamily}. 
For both functions a single trajectory \ooea gets stuck on a local optimum w.o.p. and restarts will not help to reduce the runtime to polynomial in the problem size.
On the other hand, the respective population based EAs (i.e., (1+$\lambda$)~EA for \sufsamp and the 
\mupoea for  \sufpar) can efficiently identify the global optimum by avoiding the local optima.
However, if the fitness values of the local and global optima are swapped appropriately, then the opposite effects can be shown: the single trajectory algorithms are effective while the population-based algorithms get stuck w.o.p.
Our contributions are the following. Concerning \sufsamp, we rigorously prove that the \mupoea with {\it Inverse Tournament Selection} efficiently identifies all of the many optima of the function, hence also the global one independent of which optimum is set to have the maximal value.
Concerning \sufpar~we rigorously prove that the \mupoea with inverse binary tournaments can find 
the global optimum efficiently just like the \mupoea with uniform selection.
Additionally, we show experimentally that, while the \ooea w.o.p. gets trapped in one optimum and the uniform selection \mupoea on the other, 
the  {\it Inverse Tournament Selection} algorithms identify both optima with reasonable probability, hence can optimise the function efficiently independent of which local optimum is defined to be the optimal one. 

We conclude the paper by comparing experimentally the inverse selection operators against standard uniform selection for well studied benchmark instances of two NP-Hard problems, MaxSat and the Multidimensional knapsack problem (MKP).
While for the MaxSat instances, independent runs of single trajectory algorithms outperform the population based EAs, such a strategy is detrimental for MKP\footnote{For MaxSat we implemented a \mupoea with deterministic crowding where offspring only compete with their parent. The difference in performance with $\mu$ parallel (1+1)~EAs should be minimal and not noticeable.}. 
On the other hand, for both problems the inverse selection algorithms outperform uniform selection with statistical significance i.e., better solutions are found with higher probability.

The rest of the paper is structured as follows.
In the next section we introduce the \mupoea and \muporls together with the different selection operators and the other algorithms that we will consider in the paper. We also present the \twomax and \truncatedtwomax function classes and some tools for the analyses.
In Section \ref{sec:high} we prove that standard uniform selection or higher selective pressure does not allow the efficient global optimisation of \twomax independent of the population size.
In Section \ref{sec:low} we prove that the algorithm becomes efficient if the worst individual is selected for reproduction.
In Section \ref{sec:mid} we present an experimental analysis providing evidence that inverse tournament selection with small tournament sizes is preferable to the previously considered selection operators for the optimisation of \twomax and \truncatedtwomax.
In Section \ref{sec:norestarts} we show the superiority of the inverse selection operators over uniform selection for problems where populations are essential and where restart strategies of single trajectory algorithms fail w.o.p.
In Section \ref{sec:npcomplete} we confirm our theoretical intuitions with an experimental analysis of the considered algorithms for the two NP-Hard problems.
We finish with a discussion and some conclusions. 

\section{Preliminaries}
We will consider the well-studied \mupoea described in Algorithm~\ref{alg:mupoea}. The algorithm starts by selecting $\mu$ individuals chosen uniformly at random. Then at each iteration an individual is selected from the population (parent selection) and an offspring is generated by mutating each bit of the selected individual with probability {\color{black}$p$}
\textcolor{black}{(in this paper we will use the standard mutation rate $p=1/n$).} Then, the worst individual is removed from the population breaking ties randomly. 
We will also consider its random local search variant called \muporls that flips exactly one bit per 
mutation (instead of using the mutation rate {\color{black}p}).

\begin{algorithm}[t]
\caption{\mupoea}
\label{alg:mupoea}
\begin{algorithmic}
\State Sample $\mu$ search points from $\{0,1\}^n$ independently and uniformly at random and put them into the multiset~$P$.
\For{$t \gets 1, 2, \dots$}
\State Pick $x\in P$ according to the parent selection operator.
\State Create $y$ by flipping each bit of $x$ independently with probability $1/n$.
\State Remove $\arg\min\{f(z)\,\mid\,z\in P\cup\{y\}\}$ {\color{black} from $P$}, breaking ties randomly.
\EndFor
\end{algorithmic}
\end{algorithm}

Our aim is to analyse how the parent selection operator affects the global optimisation performance of the algorithm. 
\textcolor{black}{Different definitions have been provided in the literature to refer to the {\it global optimisation} of multimodal optimisation problems.
Some researchers use the term to refer to the goal of identifying one single global optimum, while others refer to the identification of multiple optimal solutions and even to a collection of global and local optima. We refer to~\cite{PreussBook} for an overview of the available definitions in the literature.
While the insights provided in this paper apply to any of these definitions, in this paper we refer to global optimisation as the task of identifying one global optimum.
In particular, all the benchmark problems considered in this paper have only one global optimum (for \twomax either optimum may be designated to be the global one but, in order for an algorithm to identify it with high probability, it is necessary to visit both).}
We will consider the following operators enumerated from highest to lowest selective pressure: 
\begin{enumerate}
\item {\bf $K$-Tournament selection:} Select $K$ individuals uniformly at random from the population and choose the winner (i.e., best) of the tournament as parent;
\item {\bf Uniform Selection:} choose the parent uniformly at random from the population;
\item {\bf Inverse $K$-Tournament selection:} Select $K$ individuals uniformly at random and choose the loser (i.e., worst) of the tournament as parent;   
\item {\bf Inverse Elitist Selection:} choose the parent with worst fitness from the population breaking ties uniformly at random.
\end{enumerate}


\begin{algorithm}[t]
\caption{\mupoea with deterministic crowding}
\label{alg:crowding}
\begin{algorithmic}
\State Sample $\mu$ search points from $\{0,1\}^n$ independently and uniformly at random and put them into the multiset~$P$.
\For{$t \gets 1, 2, \dots$}
\State Pick $x\in P$ uniformly at random.
\State Create $y$ by flipping each bit of $x$ independently with probability $1/n$.
\State {\bf If} $f(y) \geq f(x)$  {\bf then} remove  $x$ from $P\cup\{y\}$, {\bf else} keep $P$.
\EndFor
\end{algorithmic}
\end{algorithm}

Apart from comparing different selective pressures for the \mupoea, we will also consider the use of diversity mechanisms from the literature for which good global exploration capabilities have been proven.
In particular, if the use of problem specific knowledge is not allowed, then {\it deterministic crowding} is the best diversity mechanism that has been analysed for the \twomax benchmark problem we consider in the first part of this paper.
The \mupoea using deterministic crowding is described in Algorithm~\ref{alg:crowding}. It optimises {\color{black}\twomax} in $O(\mu n \log n)$ fitness evaluations with probability at least $1-2^{-\mu}$. The difference with the standard \mupoea is that offspring only compete with their parent rather than with the whole population.
The algorithm closely resembles $\mu$ parallel (1+1)~EAs since the individuals explore the landscape independently. 
\textcolor{black}{Indeed a (1+1)~EA with a simple restart strategy optimises \twomax with the same probability if $\mu$ is the number of performed restarts or parallel runs.}
The (1+1)~EA is obtained from Algorithm \ref{alg:mupoea} by setting $\mu=1$ and preferring the offspring to the parent when breaking ties. 
Other theoretically studied diversity mechanisms are either ineffective (genotype diversity), detrimental (fitness diversity) or require problem knowledge to be made efficient (fitness sharing and clearing\footnote{Fitness sharing has only been proven to be efficient for \twomax if the knowledge that it is a unitation function is  exploited by the sharing function~\cite{FriedrichOSWECJ09,OlivetoSudholtZarges2018}. Without problem knowledge, hence using the Hamming distance for the sharing function, it has been experimentally shown that fitness sharing does not make the \mupoea efficient for \twomax~\cite{OlivetoSudholtZarges2018}. Similar argumentations also hold for Clearing which requires large population sizes and an appropriate niching radius to optimise the function efficiently~\cite{CovantesOsunaSudholt2017}.}) for \twomax~\cite{FriedrichOSWECJ09, SudholtBookChapter2019, OlivetoSudholtZarges2018, CovantesOsunaSudholt2017} hence we do not consider them in this paper. 

We are interested in the optimisation (here: maximisation) of the \twomax function,
which is defined as $\twomax:\{0,1\}^n\to\R, x\mapsto \left\lvert n/2-\ones{x}\right\rvert$. This function 
has two global optima, 
the all-ones and the all-zeros string. Leading to each optimum, there is
a so-called \emph{branch} of perfect fitness-distance correlation: 
if $\ones{x}>n/2$ then the fitness 
is $n/2$ minus the Hamming distance to the all-ones string and we say that $x$ is on the right-hand branch; analogously 
for $\ones{x}<n/2$, the Hamming distance to the all-zeros string is considered and we say that $x$ is on the so-called left-hand branch. A search point with $n/2$ ones, having 
fitness~$0$ is said to be in the \emph{valley}. The goal of our analyses is to evaluate whether the EAs can identify both optima of \twomax in expected polynomial {\it time} (i.e., the number of fitness function evaluations is polynomial with respect to the problem size $n$).
The function is depicted in Figure~\ref{fig:fitness-functions}.

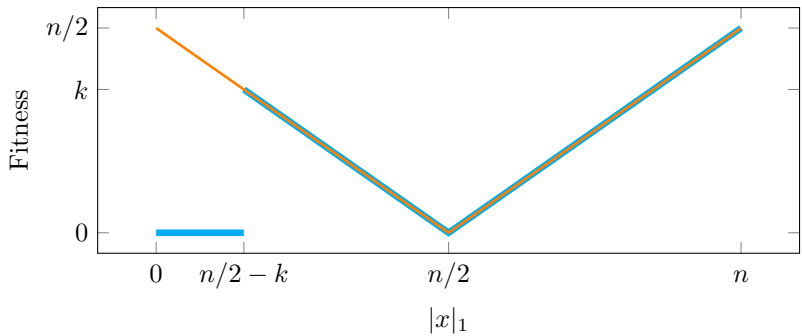
\begin{figure}[t]
\centering
\begin{tikzpicture}
\begin{axis}[ylabel=Fitness,xlabel=$|x|_1$,
	width=0.9\linewidth,
	height=0.40\linewidth,
	xtick={0,15,50,100},
	xticklabels={$0\vphantom{/2}$,$n/2-k$,$n/2$,$n\vphantom{/2}$},
	ytick={0,35,50},
	yticklabels={0,$k$,$n/2$},
	unbounded coords=jump,
	legend pos=south east,
	reverse legend,
]

\addplot[no marks,cyan,line width=2.5pt] coordinates {(0,0) (15,0) (15,nan) (15,35) (50,0) (100, 50)};
\addplot[no marks,orange,line width=1pt] coordinates {(0,50) (50,0) (100, 50)};


\end{axis}
\end{tikzpicture}
	\caption{\twomax (orange) and \truncatedtwomax (thick blue) functions.}
	\label{fig:fitness-functions}
\end{figure}

We also consider the following generalisation 
of \twomax:
\begin{equation*}
\truncatedtwomax(x) 
= \left\{ 
\begin{array}{rl}
\twomax(x) & \text{if } |x|_1 \geq n/2 -k, \\
0 & \text{otherwise.}
\end{array} \right.
\end{equation*}
The parameter $k$ defines the height of the left branch, from 0 (i.e., no left branch) to $n/2$ (i.e., \twomax).
The lower the parameter value, the easier is the optimisation of the function because
the likelihood that individuals on the local optima take over the population decreases.
The function is depicted in Figure~\ref{fig:fitness-functions}.


In our analyses we will need the following useful estimates for the probability function 
as well as the density of the binomial distribution with parameters $n$ and~$p$.

\begin{lemma}[Theorems~1.2,~1.5 and~1.6 in \cite{bollobas}]
\label{lem:bollobas-one-two}
Let $X\sim\Bin(n,p)$, $q=1-p$ and suppose $np\ge 1$. For all $h \ge 3/(nq)$ it holds that
\[
\prob{X = np + h} \le \frac{1}{\sqrt{2\pi p q n}} e^{-\frac{h^2}{2pqn} + \frac{h}{qn} + \frac{h^3}{p^2n^2}}
\]
Also, for all $h\ge 0$ where $np+h$ is an integer less than~$n$, it holds that
\[
\prob{X = np + h} \ge \frac{1}{\sqrt{2\pi p q n}} \
e^{-\frac{h^2}{2pqn} - \frac{h^3}{2q^2n^2} - \frac{h^4}{3p^3n^3} -\frac{h}{2pn} - \frac{1}{12(np+h)}  
- \frac{1}{12(nq-h)}}
\]
Finally, if $pqn=\omega(1)$, $h>0$, $h=\omega(\sqrt{pqn})$, $h=o((pqn)^{2/3})$ as well as $np+h\le n$ then
\[
\prob{X\ge np + h} = (1\pm o(1)) \frac{\sqrt{pqn}}{h\sqrt{2\pi}} e^{-\frac{h^2}{2pqn}}
\]
\end{lemma}

As a corollary from the previous lemma, we obtain  the following
asymptotically tight estimate for the case of the symmetrical binomial distribution, \ie,  
$p=1/2$: 
\begin{corollary}
\label{cor:prob-symm-binomial}
Let $X\sim\Bin(n,1/2)$. Then for all $h=o(n^{2/3})$ such that $n/2+h\in \{n/2,\dots,n\}$ it holds that 
\[
\prob{X = n/2 + h} = (1\pm o(1)) \frac{\sqrt{2}}{\sqrt{\pi n}} e^{-\frac{2h^2}{n}}
\]
as well as
\[
\prob{X\ge n/2 + h} = (1\pm o(1)) \frac{\sqrt{n}}{h\sqrt{8\pi}} e^{-\frac{2h^2}{n}}
\]
\end{corollary}

In particular, the corollary will be used to analyse the maximum order statistic with respect to the number 
of one-bits in a uniform 
population of~$\mu$ individuals. Related results can be found in \cite{LaillevaultGECCO15}; however, 
we find the bounds based on \cite{bollobas} presented above more convenient to use.

The following two theorems,  an additive variant of the 
classical Chernoff-Hoeffding inequality and the artificial fitness levels theorem with tail 
probabilities, will be used in Section~\ref{sec:low} for the analysis of the \mupoea.
\begin{theorem}[Th.~1.11 in \cite{DoerrBookChapter}]\label{thm:cheradd}
Let $X_1,\ldots,X_n$ be independent random variables taking values in $[0,1]$. 
Let $X=\sum_{i=1}^{n}X_i$. Then for all $\lambda\geq 0$, 
$\prob{X\geq E[X] + \lambda}\leq \exp(-\frac{2\lambda^2}{n})$ and $\prob{X\leq
E[X] - \lambda}\leq \exp(-\frac{2\lambda^2}{n})$.
\end{theorem}

Finally, we  make use of the following artificial fitness levels theorem with tail probabilities.
\begin{theorem}[Theorem~2 in \cite{WittIPL2014}]
\label{theo:fitness-tail}
Consider an algorithm~$\mathcal{A}$ maximising some 
function~$f$ and a partition of the search space into non-empty sets $A_1,\dots,A_m$. Assume 
that the sets form an $f$-based partition, \ie, 
for $1\le i<j\le m$  and all $x\in A_i$, $y\in A_j$ it holds $f(x)<f(y)$. 
We say that $\mathcal{A}$  is in $A_i$ or on level~$i$ if the best search point created so far 
is in $A_i$. 

If $p_i$ is a lower bound 
on the probability that a step of~$\mathcal{A}$ leads from level~$i$ to some higher level, 
independently 
of previous steps, then for any $\delta>0$, 
the first hitting time of~$A_m$, starting from level~$k$, is at most
\[
\sum_{i=k}^{m-1} \frac{1}{p_i} + \delta
\]
with probability at least $1-e^{-\frac{\delta}{4}\cdot \min\{\frac{\delta}{s}, h\}}$, for any 
finite 
$s\ge \sum_{i=k}^{m-1} \frac{1}{p_i^2}$ 
and $h = \min\{p_i \mid i=k,\dots,m-1\}$. 
\end{theorem}

For the remainder of the paper we will use 
\emph{with overwhelming probability} (w.o.p.) to mean with probability $1-2^{-n^{\Omega(1)}}$.


\section{Too High Selective Pressure Fails}\label{sec:high}
It is known from previous work that the standard \mupoea  (using uniform parent selection) with high probability 
fails to find both optima of \twomax if $\mu=o(n/\log n)$~\cite{FriedrichOSWECJ09}. 
More precisely, with probability $1- o(1)$ 
there will be within any polynomial amount of time no step where the current population contains 
both optima. The ideas behind the analysis in \cite{FriedrichOSWECJ09} can briefly be summarized as follows:
\begin{itemize}
\item With overwhelming probability the initial population contains none of the two optima.
\item As every generation creates exactly one new individual, there is a first point in time where 
exactly one of the optima, say the all-ones string, is created and added to the population.
\item The take-over time for the all-ones string is $O(\mu\log \mu)=o(n)$ if the all-zeroes string is not found 
in between. With probability $1-o(1)$ the take-over time is still $o(n)$.
\item 
Pessimistically assuming all individuals from the branch leading to the all-zeros string to have only 
a single-one bit, the probability of creating the all-zeros string by mutation is still at most~$1/n$. 
With high probability this step does not happen during the take-over time.
\item 
Hence, with high probability after $o(n)$ steps the whole population consists of copies 
of one optimum only, and in no step before has 
the other optimum been in the population. The probability of creating 
the opposite optimum by mutation is $1/n^n$, which w.o.p. 
does not happen within 
any polynomial number of steps.
\end{itemize}

Our analysis will considerably improve the result from \cite{FriedrichOSWECJ09}. First of all, we show 
that even with arbitrarily large polynomial population sizes~$\mu$ the algorithm will fail to find both optima. Second, 
we show that the whole population will be on one branch of \twomax in very early stages of the optimisation, \ie, 
soon after initialisation. Hence, diversity is lost much earlier than in the final improving steps
that are pessimistically considered in \cite{FriedrichOSWECJ09}.

In a nutshell, the proof idea for the following theorem is to consider the uniform initialization 
procedure of the \mupoea
and to analyse the probability of so-called outliers. More formally, an outlier is an individual~$x^*$  
from the initial population that has 
the largest distance $\card{\ones{x^*}-n/2}$ from the points in the valley $n/2$ and also has 
a significantly larger distance to the valley than all other individuals. Again, an argument 
about the take-over time then shows that with high probability the whole population will be descendants of the outlier after 
$O(\mu\log \mu n^{\epsilon})$ steps. Here we use the family tree technique by Witt \cite{WittFamily} to show 
that offspring of other individuals would have to 
make an exceptionally large progress within the take-over time to catch up to the level of the best. After the 
population has been taken over by descendants of the outlier, 
 steps that 
mutate individuals to accepted individuals from the opposite branch are superpolynomially unlikely  
due to the large distance of the outlier's fitness to the valley.

\begin{theorem}
\label{theo:negative-result-standard-muone}
Let $\mu\le n^{k}$ for an arbitrary constant $k\ge 0$ as well 
as $\mu=\omega(1)$ and consider the \mupoea on \twomax. 
Let $\epsilon>0$ be an arbitrary constant. 
Then with probability $1-o(1)$, all individuals are on one branch 
before the $\twomax$-value reaches $n/2+c\sqrt{n}(\log n+\log \mu)$, for 
some sufficiently large constant~$c>0$, and 
the worst $\twomax$-value  in the population 
is $n/2+\Omega(n^{1/2-\epsilon})$ 
at the time one branch dies out. This implies time $2^{\Omega(n^{1/2-\epsilon})}$ 
with probability at least $1-2^{-\Omega(n^{1/2-\epsilon})}$ to create
an individual from that branch again.
%
%
\end{theorem}

To prove Theorem~\ref{theo:negative-result-standard-muone}, we will essentially look into random variables~$X_i$, 
where $1\le i\le \mu$,  
following the Binomial distribution with parameters $n$ and~$1/2$. These variables describe the number 
of one-bits in the $\mu$ uniformly initialized individuals. Since we are only interested in the distance of this 
number to the valley $n/2$, we introduce 
$D_i\coloneqq \abs{ \onenorm{x_i} - n/2}$ for $i\in\{1,\dots,\mu\}$ and relate these 
to~$X_i$. 
The probability distribution 
of $D_i$ equals the one of $X_i-n/2$, conditioned on $X_i\ge n/2$. Also, $1/2 \le  \Prob(X_i\ge n/2)\le 1/2+O(1/\sqrt{n})$, 
so for $a\ge 0$ 

\[(2-o(1)) \prob{X_i\ge n/2+ a} \le   \prob{D_i \ge a} \le (2+o(1))\prob{X_i\ge n/2 +a}.\]

In the following, we will analyse the probability of an outlier, \ie, an individual 
that has a $D$-value significantly larger than the standard deviation 
$\sigma=\sqrt{n/4}$  of the binomial distribution $\Bin(n,1/2)$.
Let 
$D_{(1)}\ge \dots \ge D_{(\mu)}$ be  the order statistics of the~$D_i$ (for convenience, index $(1)$ 
denotes the largest order statistic here, in contrast to the customary use in the literature). We shall 
also use the random variables $(i)$, where $i\in\{1,\dots,\mu\}$, to denote the index of the $i$th order statistic.

We recall the assumptions  $\mu\le n^k$ for some constant $k\in\N$ and $\mu=\omega(1)$
from the main theorem, where the latter is mostly for convenience of proof. Our first result 
shows a lower bound on $D_{(1)}$, the largest order statistic, that holds with high probability.

\begin{lemma}
\label{lem:d-one-greater-xstar}
There is a $\ell^*=\sqrt{n \bigl((1/2-1/(8k))\ln \mu-\Theta(\log\log \mu)\bigr)}$ such that 
$\prob{D_{(1)}\ge \ell^*} = 1-o(1).$
\end{lemma}

\begin{proof}
Let $\ell^*$ be the solution for~$h$ 
of the equality 
\[
\frac{\sqrt{n}}{h\sqrt{8\pi}}e^{-2h^2/n} = \frac{\mu^{-1+1/(4k)}}{2}.\]
We first prove the asymptotic expression given for~$\ell^*$ in the statement of the lemma. 
We write \[\ell^*=\sqrt{n ((1/2-1/(8k))\ln \mu- c(n)\ln \ln\mu)}\] 
for a function $c(n)$, \ie, $c(n)$ stands for the $\Theta(\log\log \mu)$-term in the claimed 
expression   
divided by $\ln\ln \mu$. 
Then, 
\begin{align*}
 \frac{\sqrt{n}}{\ell^*\sqrt{8\pi}}e^{-2(\ell^*)^2/n}  &  = 
\frac{\sqrt{n} \mu^{-1+1/(4k)} e^{2c(n)\ln\ln \mu}}{\ell^*\sqrt{8\pi} }   = 
	\frac{\sqrt{n}(\ln\mu)^{2c(n)} }{\sqrt{n \bigl((1/2-1/(8k))\ln \mu-c(n)\ln\ln \mu\bigr)} \sqrt{8\pi}}  \mu^{-1+1/(4k)} \\
	& = \frac{1}{\sqrt{8\pi\left(\left( \frac{1}{2}-\frac{1}{8k}\right)(\ln \mu)^{1-4c(n)}-\frac{c(n)\ln\ln \mu}{(\ln \mu)^{4c(n)}}\right)} }  \mu^{-1+1/(4k)} 
  \end{align*}
	Now, 
	the fraction equals $1/2$ for some choice of $c(n)$ that converges to the constant~$1/4$, as can be seen 
	from solving 
	\[
	\sqrt{8\pi\left(\frac{1}{2}-\frac{1}{8k}\right)(\ln \mu)^{1-4c(n)}} = 2
	\]
	for $c(n)$ and exploiting that $\frac{c(n)\ln\ln \mu}{(\ln \mu)^{4c(n)}}=o(1)$ thanks to our assumption $\mu=\omega(1)$.

By Lemma~\ref{cor:prob-symm-binomial}, we have for any of the random variables~$X_i \sim \Bin(n,1/2)$ 
that 
$\prob{X_i\ge n/2+\ell^*} \ge (1-o(1))\mu^{-1+1/(4k)}/2$ and therefore
$\prob{D_i \ge \ell^*} \ge (1-o(1))\mu^{-1+1/(4k)}$. Using the independence of the random variables, the probability 
at there is at least one~$i\in\{1,\dots,\mu\}$ with this property is therefore at least
\[
1-\left(1-\mu^{-1+1/(4k)}\right)^\mu = 1-o(1).
\]
\end{proof}

While the previous lemma established the existence of a high extremal 
value for the $D_i$, the next lemma shows that the maximum order 
statistic with high probability actually is an outlier, more precisely, 
that it is almost $\sqrt{n}$ larger than the second-largest $D$-value.

\begin{lemma}
\label{lem:order-statistic-gap}
With probability $1-o(1)$, there exists $\delta^*=O(n^{1/2-\epsilon})$ for an arbitrary constant~$\epsilon>0$ 
such that $D_{(1)}-D_{(2)}\geq \delta^*$.
\end{lemma}

To prove Lemma~\ref{lem:order-statistic-gap}, we need to show for all $D_i\neq D_{(1)}$ that they are 
significantly smaller than $D_{(1)}$. However, this does not imply the lemma yet since 
there may be more than one variable taking the maximum value. Hence, 
we first prove the following auxiliary statement, showing that 
the largest order statistic is unique with high probability.

\begin{lemma}
\label{lem:order-statistic-gap-aux}
Assume $D_{(1)}\ge \ell^*$ and let $i$ be an index such that $D_i=D_{(1)}$. 
Then with probability $1-o(1)$ it holds that $D_{j}<D_{i}$ for all $j\neq i$. 
\end{lemma}

\begin{proof}
The main idea of the proof is to exploit that the first bound (on the probability value) and the second bound (on the 
distribution) in Corollary~\ref{cor:prob-symm-binomial} 
are within a ratio of $\Theta(h/n)$, making it unlikely for any other individual to hit exactly the value of~$D_{(1)}$. 
For a moment, we consider the random $D_i$ without the condition 
$D_{(1)}\ge \ell^*$. 
According to Lemma~\ref{lem:d-one-greater-xstar}, we have 
 $\ell^*=\Theta(\sqrt{n\log \mu})$.
 We already know that $\Prob(D=\ell^*)=(2+o(1))\Prob(X_i\ge n/2+\ell^*)$ and obtain 
from the first inequality in Corollary~\ref{cor:prob-symm-binomial} that
\begin{align*}
\Prob(D_i = \ell^*) = O\,\Bigl(\frac{\polylog \mu}{\mu^{1-1/(4k)}\sqrt{n}}\Bigr) 
 = 
O\,\Bigl(\frac{\polylog \mu}{\mu n^{1/4}}\Bigr) ,
\end{align*}
where we used $\mu^{1/(4k)}\le n^{1/4}$. Also, for values $k\ge \ell^*$ 
we clearly have $\Prob(D=k)\le \Prob(D=\ell^*)$ by the monotonicity of the binomial 
distribution.

We now consider the generation of the $\mu$ random 
$D_i$-values as a stochastic process, assuming that 
$D_i$ is generated at time~$i$. Note that our condition 
$D_{(1)}\ge \ell^*$ is equivalent to that there exists a (random) $i\in\{1,\dots,\mu\}$ such that 
$D_i\ge \ell^*$. Let $T$ denote the smallest such~$i$. For $j<T$ is clearly 
holds $D_j<\ell^*\le D_{(1)}$. After time~$T$, the 
condition $D_{(1)}\ge \ell^*$ is satisfied so that each of the $D_{T+1},\dots,D_{\mu}$ is conditioned 
on being at most $D_{(1)}$. Since even $\prob{D_j\le \ell^*} = 1-o(1)$, the conditional probability 
is only by a factor $1+o(1)$ larger than the unconditional one.
 Hence, it is sufficient to show that 
the values $D_{T+1},\dots,D_{\mu}$ unconditionally  with high probability are less than $D_{(1)}$. 

By the above calculations, 
for all $i>T$, we get 
\[
\Prob(D_i=D_{(1)}) = 
 O\,\Bigl(\frac{\polylog \mu}{\mu n^{1/4}}\Bigr).
\]
Now a union bound over the $\mu-T$ variables yields the desired result.
\end{proof}

We now prove the outlier property of the largest order statistic. 
\begin{proofof}{Lemma~\ref{lem:order-statistic-gap}}
We recall that $\ell^*=\Theta(\sqrt{n\log \mu})$. 
By using the second inequality of Corollary~\ref{cor:prob-symm-binomial},
we note that for $\delta^*=O(n^{1/2-\epsilon})$ and any~$i\in\{1,\dots,\mu\}$ it holds that
$C\prob{D_i\geq\ell^*}=(1\pm o(1))\prob{D_i\geq\ell^*-\delta^*}$,
where
\begin{equation*}
C=\left(1-\frac{\delta^*}{\ell^*}\right)^{-1}
\exp\left(\frac{4\ell^*\delta}{n}-\frac{2\delta^2}{n}\right)
=1+O(n^{-\epsilon})\enspace.
\end{equation*}
Now,
\begin{align*}
&\prob{D_{(1)}-\delta\leq D_i < D_{(1)} \mid D_{(1)}\geq\ell^*\wedge 
D_i<D_{(1)}}\leq\prob{D_i\geq\ell^*-\delta\wedge D_i<\ell^*\mid D_i<D_{(1)}}\\
&\leq\left(1-o(1)\right)^{-1}\left(1-\prob{D_i\geq\ell^*}
-\left(1-\prob{D_i\geq\ell^*-\delta\right)\right)}\\
&=(1+o(1))(C-1)\prob{D_i\geq\ell^*} =O(1/\mu)\mu^{1/(4k)} n^{-\epsilon} = O(\mu^{-1}n^{-\epsilon/2})\enspace, 
\end{align*}
where the first inequality follows from exploiting the condition
that $D_{(1)}\geq\ell^*$
and the fact that $\Prob([x-a,x])$ is decreasing in $x$
for any positive constant $a$ and $x-a$ being larger than the mode.
In the second inequality we used $\Prob(A\mid B)\leq\Prob(A)/\Prob(B)$
and
\begin{equation*}
\Prob(D<D_{(1)})\geq
\Prob(D_{(1)}\geq\ell^*)\Prob(D<D_{(1)}\mid D_{(1)}\geq\ell^*)=1-o(1)\enspace,
\end{equation*}
which follows from Lemma~\ref{lem:d-one-greater-xstar} and Lemma~\ref{lem:order-statistic-gap-aux}.
The equalities then follow from bound on $C$ derived above, the 
fact $\prob{D_i\ge \ell^*}\le \mu^{-1+1/(4k)}$ and sufficiently large~$k$.

Hence, we have that the unconditional probability
\begin{equation*}
\Prob\left(D_{(1)}-\delta\leq D < D_{(1)}\right)
\leq\frac{o(1/\mu)}{(1-o(1))}
\end{equation*}

Now, by a union bound, the probability that there exists some $j\neq (1)$ such that 
$D_{(1)}-\delta\leq D_j < D_{(1)}$ is $o(1)$. In particular, the probability 
of $D_{(2)}$ being in this interval is $o(1)$.
\end{proofof}

Lemma~\ref{lem:order-statistic-gap}, which establishes the outlier in the initial population 
is now crucial to show that the outlier will take over the whole population with high probability 
before the rest of the population can catch up to the outlier. The following 
lemma makes this take-over effect formal.

\begin{lemma}
\label{lem:die-out-because-of-gap}
If the best individual in a population optimising \twomax has a fitness value 
$\Omega(n^{\epsilon})$ larger than the second best for some 
$\epsilon=\Omega(1)$, then the population will be taken over by individuals from 
the best solution's branch whose fitness values are all at least as good as the 
initial best fitness in $n^{\epsilon/4} \mu \log{\mu} $ generations w.o.p. 
\end{lemma}
\begin{proof}
Using a standard analysis of take-over times (\eg, \cite{FriedrichOSWECJ09,WittFamily}) it takes an expected number of  
$\Theta(\mu \log \mu)$ 
generations until the population is taken 
over by solutions from the best individual's branch, unless an equal or better search point 
on the opposite branch is sampled. Using Markov's inequality and analyzing the probability 
of failing to take over within $n^{\epsilon/4}$ phases of length $\Theta(\mu \log \mu)$ each, 
we obtain that the take-over time is $O(n^{\epsilon/4} \mu \log{\mu})$ w.o.p. 

We now use the well-known family tree technique \cite{WittFamily} for the analysis of the 
\mupoea to show that the offspring of any other individual than the initially best 
with high probability cannot catch up to the fitness of the initially best. 
The family tree of any initial solution will not 
reach a depth larger than $n^{\epsilon/2}$ in 
$n^{\epsilon/4} \mu \log{\mu}$ generations w.o.p. 
according to Lemma~2 in \cite{WittFamily}. 
Let us call a path with any node which has Hamming 
distance at least $3 n^{\epsilon/2}$ to its root a 
{\it bad path}. The probability of observing a particular bad path 
of length $\ell$ in a tree of depth smaller than 
$n^{\epsilon/2}$ is at most $(1/\mu)^\ell e^{- 3n^{\epsilon/2}}$. Using 
a union bound over all paths of length at most $n^{\epsilon/2}$ 
that can be observed in $n^{\epsilon/4} \mu \log{\mu}$ generations of the \mupoea, 
we upper bound the probability of sampling a solution which has $3n^{\epsilon/2}$ more 1-bits than its root by,
\[
\sum^{n^{\epsilon/2}}_{\ell=1} \binom{n^{\epsilon/4} \mu \log{\mu}}{\ell} (1/\mu)^\ell e^{- 3n^{\epsilon/2}}\leq n^{\epsilon/2} e^{n^{\epsilon/3}} e^{- 3n^{\epsilon/2}}=2^{-\Omega(n)}.
\]

The probability that a bad path will not be observed in the family tree 
of any solution in the initial population can be shown to be exponentially 
small by using a union bound once again. Hence, no offspring of the individuals 
different from the best one will reach the fitness of the best with high probability.
\end{proof}

We put everything together to prove our main result:

\begin{proofof}{Theorem~\ref{theo:negative-result-standard-muone}}
According to Lemma~\ref{lem:order-statistic-gap}, 
$D_{(1)}-D_{(2)}=\Omega(n^{1/2-\epsilon})$ with probability $1-o(1)$. 
We assume this to happen. Hence, Lemma~\ref{lem:die-out-because-of-gap} can be applied, 
such that  after 
$n^{1/8} \mu \log{\mu} $ generations one branch will have died with probability $1-o(1)$. Moreover, 
the lemma shows that the whole population after this time will be at least as good as the initially best one, which proves 
that the worst  $\twomax$-value will be at least $n/2+D_{(1)} = n/2+\Omega(n^{1/2-\epsilon})$  at the time one branch dies out. 

Finally, using the family tree technique in the same way as 
for Lemma~\ref{lem:die-out-because-of-gap}, we note 
 that the best \twomax-value in the population 
will not increase by more than $n^{1/8}\log \mu$ in $n^{1/8} \mu \log{\mu} $ generations 
with probability~$1-o(1)$. Using that 
 $D_{(1)}=O(\sqrt{n}(\log \mu+\log n))$ with probability $1-o(1)$ (which follows 
from simple Chernoff and union bounds), we obtain 
that with probability $1-o(1)$,  the best $\twomax$-value is $O(n/2+c\sqrt{n}(\log n+\log \mu))$ at the time 
one branch dies out, as claimed in the theorem statement.
%
%
\end{proofof}


We can straightforwardly 
generalize Theorem~\ref{theo:negative-result-standard-muone} to \textsc{TruncatedTwoMax} if the maximum fitness of the 
truncated branch lies well above the best fitness reached after a branch has taken over. The theorem considers a \twomax-value of $n/2+c \sqrt{n}\log (\mu n)$, so 
in principle we only have to choose $k$ slightly larger than the offset $c \sqrt{n}\log (\mu n)$. For convenience, we formulate 
our corollary with $k = \omega(\sqrt{n}\log n)$.

\begin{corollary} 
Let $\mu\le n^{k'}$ for an arbitrary constant $k'\ge 0$ as well 
as $\mu=\omega(1)$ and consider the \mupoea on \textsc{TruncatedTwoMax} with $k = \omega(\sqrt{n}\log n)$. Let $\epsilon>0$ be 
an arbitrary constant. Then with probability 
$1/2-o(1)$ the whole population is on the truncated branch before the global optimum is reached, and 
conditional on this event the optimisation time is $2^{\Omega(n^{1/2-\epsilon})}$ \wop
\end{corollary}


\subsection{RLS Mutation}

We can obtain a result analogous to Theorem~\ref{theo:negative-result-standard-muone}
also for the 
\muporls which differs from the \mupoea only in the mutation operator. Basically, 
the only difference is in the analysis of take-over times, which originally 
refer to the time until the whole population descends from a certain individual and 
is at least as good as it. For the \mupoea using 
standard-bit mutation, 
the $O(\mu \log \mu)$ bound is based on the fact that an individual can produce 
a copy of itself with probability $(1-1/n)^n \approx e^{-1}$. This is not the case 
for the one-bit mutation of the \muporls, which either improves or degrades individuals' fitness. Nevertheless, we can use the same type of analysis 
if the best and second-best individual are sufficiently large apart and the offspring of the best 
do not become worse than the offspring of the second-best within a certain period of time.

More formally, let $x^*$ be the initially best individual and let after~$t$ generations 
the set $D_t$ denote the individuals in the population at time~$t$ that are offspring of~$x^*$. 
Given $\card{D_t}=k<\mu$, the probability of choosing from $D_t$ is $k/\mu$ and, if this happens and the offspring 
is accepted, it holds $\card{D_{t+1}}=k+1$. Hence, assuming all these offspring to be accepted, the time 
until the whole population descends from $x^*$ has expected value $O(\mu \log \mu)$. Note that this does not imply 
that all descendants are at least as good as~$x^*$ in fitness.

Let $R^*$ be the other $\mu-1$ individuals from the initial population. 
Offspring of~$x^*$ are guaranteed to be accepted in the proposed time of expected $O(\mu\log \mu)$ generations 
if 
\begin{itemize}
\item no offspring of individuals in $R^*$ reaches fitness at least $(D_{(1)} + D_{(2)})/2$
\item no offspring of~$x^*$ has fitness less than $(D_{(1)}+D_{(2)})/2$
\end{itemize}

For either event to fail, the  progress must be at least~$(D_{(1)}- D_{(2)})/2$ in at least one lineage 
of one of the~$\mu$ family trees stemming from the individuals of the initial population. Using Lemma~\ref{lem:die-out-because-of-gap}
to bound the depth of each family tree by $n^{\epsilon/2}$ and observing that the progress along any lineage 
due to one-bit mutation is bounded by the length of the lineage, we obtain a progress of at most 
$n^{\epsilon/2}$ with probability~$1-o(1)$ during the takeover time. This is clearly smaller than 
the bound $(D_{(1)}- D_{(2)})/2 = n^{\epsilon/2}$ that results from uniform initialization 
according to Lemma~\ref{lem:order-statistic-gap} with probability~$1-o(1)$. Hence, we have obtained the following result:

\begin{theorem}
\label{theo:negative-result-standard-muonerls}
Let $\mu\le n^{k}$ for an arbitrary constant $k\ge 0$ as well 
as $\mu=\omega(1)$ and consider the \muporls on \twomax. 
With probability $1-o(1)$, all individuals are on one branch 
before the $\twomax$-value reaches $n/2+c\sqrt{n}(\log n+\log \mu)$, for 
some sufficiently large constant~$c>0$, and 
the worst $\twomax$-value  in the population 
is $n/2+\Omega(n^{1/2-\epsilon})$ 
at the time one branch dies out. If this happens, individuals 
from the other branch cannot be generated any more.
%
%
\end{theorem}

%
%

\section{Low Selective Pressure Succeeds}\label{sec:low}
In this section we will show that reducing the selection pressure to
its minimum by always picking the worst parent for reproduction makes the \mupoea
efficient at locating both optima of \twomax with reasonable population sizes.
In the following section we will argue why such a strategy is not sufficient for efficient global optimisation in general, and propose necessary corrections.
We begin by analysing the standard bit mutation version, the \mupoea, and afterwards we consider the 
RLS version. At the core of both proofs is the analysis of the stochastic process $D_t$ which keeps 
track of the difference between the number of individuals with more $1$-bits and more $0$-bits. 
The crucial observation is that, with inverse elitist parent selection, $D_t$ performs a fair 
random walk which starts close to $0$ and, for large enough $\mu$, cannot reach w.o.p. any of the two 
boundaries ($\mu$ or $-\mu$) before both optima are found.


\begin{theorem}\label{thm:uppersbm}
The \mupoea with $\mu=\Omega(n^{1+2\epsilon})$ and inversely elitist parent selection 
finds both 
optima of \twomax in $O(\mu n^{1+\epsilon})$ fitness function evaluations with 
probability 
$1-e^{-\Omega(n^{\epsilon})}$ for any  positive constant $\epsilon$ such that $1 
- \epsilon = \Theta(1)$.
	\end{theorem}
	\begin{proof}
	We consider the population $P_t$ at time (\ie, iteration) $t$ and its 
mutually exclusive 
sub-populations $P_{0}^{t}$ ($P_{1}^{t}$) which consist of individuals with 
more 0-bits (1-bits) in their bitstrings. Let $D_t:=|P_{0}^{t}|-|P_{1}^{t}|$ 
denote 
the difference between the sizes of the sub-populations at time $t$ and 
$\mathcal{T}$ 
be the smallest $t$ such that $P_t$ consists only of $1^n$ or $0^n$ 
bit-strings. 
We note here that $|D_t|=\mu$ if and only if  $P_t=P_{0}^{t}$ or 
$P_t=P_{1}^{t}$. 
Consequently, $|D_{\mathcal{T}}|< \mu $ implies that both optima are found at 
time $\mathcal{T}$.

 First of all, we will prove the upper bound $\mathcal{T} < \mu e 
n^{1+\epsilon}$ 
which holds with probability at least $1-e^{-\Omega(n^\epsilon)}$ for any 
constant $\epsilon>0$. 
Then, in order to show that $|D_t|<\mu$ for any  $t \in [\mu e n^{1+\epsilon}]$, 
we will first show that, $|D_0| < \mu^{\frac{1+\epsilon}{2}}$ with probability 
at 
least $1-e^{-\Omega(\mu^{2\epsilon})}$, and then finally show that for any 
$t \in [\mu e n^{1+\epsilon}]$ the probability that $|D_t|-|D_0|> \mu 
-\mu^{\frac{1+\epsilon}{2}}$ 
is at most  $e^{-\Omega(n^\epsilon)}$. Using a union bound over (1) the 
probability 
that $\mathcal{T}>\mu e n^{1+\epsilon}$, (2) the probability that the initial 
difference 
$|D_0| > \mu^{\frac{1+\epsilon}{2}}$ (3) all the probabilities 
$|D_t|-|D_0|> \mu -\mu^{\frac{1+\epsilon}{2}}$ for $t \in [\mu e 
n^{1+\epsilon}]$; 
we will obtain our claim.

Given that the best individual in the population has fitness value~$i$, 
the probability that a solution with fitness at least $i+1$ will be sampled in 
the next iteration is at least $(n-i)/(\mu en)$. Since 
$\sum_{i=n/2}^{n-1}p_{i}^{-1}\leq \sum_{i=n/2}^{n-1}\frac{\mu en}{n-i}\leq \mu 
e n \ln{n}$, we can use Theorem~\ref{theo:fitness-tail} to bound the 
probability that $\mathcal{T}$ is at least $\mu e n^{1+\epsilon}$ from 
below.
\begin{align*}
& \prob{\mathcal{T}>\mu e 
n^{1+\epsilon}}\leq \prob{\mathcal{T}>\sum_{i=n/2}^{n-1}p_{i}^{-1} + \mu  
e n^{1+\epsilon}- \mu e n \ln{n}}\\
&\leq \exp{\left(-\frac{ \mu e n \left( n^{\epsilon}-\ln{n} \right) 
}{n}\right)}= \exp(-\Omega(\mu n^{ \epsilon}))
\end{align*}
%
%
%
%
%
  From now on, we will
assume  that $\mathcal{T}\leq\mu e n^{1+\epsilon}$ and account for the case of 
$\mathcal{T}>\mu e n^{1+\epsilon}$ as part of the overall failure probability 
$e^{-\Omega(n^\epsilon)}$.

	Next, we will show that, with probability $e^{-\Omega(n^\epsilon)}$, 
$\mu e 
n^{1+\epsilon}$ iterations will not be enough for one of the branches to take over the population.  $D_t$ changes only when a solution from $P_{0}^{t}$ 
($P_{1}^{t}$) is selected for mutation and then an individual from the 
opposite sub-population $P_{1}^{t}$ ($P_{0}^{t}$) is selected for removal. 
Since 
the parent and the environmental selections are identical and independent, the 
probability $p^+$ of selecting from $P_{0}^{t}$ as parent and from $P_{1}^{t}$ 
for removing is equal to the probability $p^-$ of selecting the parent from 
the $P_{1}^{t}$ and then removing from $P_{0}^{t}$.  Moreover, all parent 
candidates have the 
same probability of being improved and the same probability of being copied, 
thus, 
\[
\prob{D_{t+1}=D_{t}+2 \mid D_{t+1}\neq D_{t}}= \prob{D_{t+1}=D_{t}-2 \mid 
D_{t+1}\neq D_{t}}=1/2 
\]
Let a relevant step be any $t$ such that $D_t\neq D_{t+1}$ and $\pi_k \in [T]$
 denote the actual iteration number of the $k$th relevant step.
 Then,  $D_{\pi_k+1}-D_{\pi_k} \sim 4 \cdot 
\Bernoulli(\frac{1}{2})-2$ and consequently $D_{\pi_k} - D_0 \sim 4 \cdot 
\Bin(k,\frac{1}{2})-2\cdot k$. 
		Since the initial population consists of individuals sampled 
uniformly at random, the initial difference $D_0 \sim 2\cdot \Bin(\mu,1/2) - 
\mu$ 
and $\prob{D_0 > \mu^{\frac{1+\epsilon}{2}}}=\exp\left( 
-\Omega(\mu^{\epsilon})\right)$. Symmetrically we also have $\prob{D_0 < 
-\mu^{\frac{1+\epsilon}{2}}}=\exp\left( -\Omega(\mu^{\epsilon})\right)$ and 
can use the union bound to obtain $\prob{|D_0| < 
\mu^{\frac{1+\epsilon}{2}}}=\exp\left( -\Omega(\mu^{\epsilon})\right)$. 
Therefore, for any relevant step $k$ and the random variable $X \sim 
\Bin\left(k,\frac{1}{2}\right)$; 
\begin{align*}
			&\prob{|D_{\pi_k}|=\mu \quad \big| \quad |D_0| < 
\mu^{\frac{1+\epsilon}{2}}}\leq \prob{|D_{\pi_k}|-|D_0|\geq 
\mu-\mu^{\frac{1+\epsilon}{2}}}\leq \prob{D_{\pi_k}-D_0\geq 
\mu-\mu^{\frac{1+\epsilon}{2}}}\\
				&= \prob{4\cdot X -2\cdot k\geq 
\mu-\mu^{\frac{1+\epsilon}{2}}} = \prob{ X \geq 
\frac{\mu-\mu^{\frac{1+\epsilon}{2}}+2k}{4}}\\
								&= \prob{ X 
\geq 
\frac{\mu-\mu^{\frac{1+\epsilon}{2}}}{4}+E[X]}\leq \exp\left( 
 - \frac{2 \cdot 
\left(\mu-\mu^{\frac{1+\epsilon}{2}}\right)^2}{16 k}\right)
\end{align*}
where in the last step we used Theorem~\ref{thm:cheradd}. Taking the assumptions 
$1-\epsilon=\Theta(1)$, $\mu=\Omega(n^{1+2\epsilon})$ and $k \leq \mu e 
n^{1+\epsilon}$, it holds that
\[
 \exp\left( 
 - \frac{2 \cdot 
\left(\mu-\mu^{\frac{1+\epsilon}{2}}\right)^2}{16 k}\right)=\exp\left( 
 - \Omega\left(\frac{\mu^2 }{k}\right)\right)=\exp\left( 
 - \Omega\left(\frac{\mu }{n^{1+\epsilon}}\right)\right)=\exp\left( 
 - \Omega\left(n^{\epsilon}\right)\right).
\]

Note that the above probability is the probability that a single $k \in [\mu e 
n^{1+\epsilon}]$ 
would have $|D_{\pi_k}|\geq \mu$. 
To get our final results, we will use the union bound to aggregate the 
probabilities 
of failure for all such $k$  along with the probability that $\mathcal{T} > e 
\mu n^{1+\epsilon} $ 
and that $|D_0|>\mu^{\frac{1+\epsilon}{2}}$.
\begin{align*}
			&\prob{|D_{\mathcal{T}}|\neq \mu} > 1 - e \mu 
n^{1+\epsilon}  \exp\left(-\Omega\left(n^{\epsilon}\right)\right)- 
\exp\left(-\Omega\left(n^{\epsilon}\right)\right) - 
\exp\left(-\Omega\left(\mu n^{\epsilon}\right)\right)\\ &= 1- 
\exp\left(-\Omega\left(n^{\epsilon}\right)\right).
\end{align*}
\end{proof}	
\begin{theorem}
	 For any arbitrarily small constant $\epsilon>0$, the
	  \muporls with inversely elitist parent selection and 
$\mu=\Omega(n^{1+\epsilon})$ finds both optima of \twomax in $O(\mu n\log{n})$ fitness 
function evaluations with probability $1-e^{-\Omega(n^{\epsilon})}$.
	\end{theorem}
	\begin{proof}
	We will follow the same proof idea and the definition of relevant step 
as in Theorem~\ref{thm:uppersbm}, however, since RLS cannot create copies, the 
number $\mathcal{K}$ of relevant steps in a run  is bounded from above by 
$\mu\cdot n$ because a solution is improved in each relevant step and each 
individual can be improved at most $n$ times. We will denote the actual 
iteration number of the $k$th relevant step as $\pi_k \in [T]$.	At every 
relevant step $k$, the difference $D_k:=|P_{0}^{\pi_k}|-|P_{1}^{\pi_k}|$ either 
increases or decreases by two with equal probability $1/2$. Formally, 
$D_{k+1}-D_{k} \sim 2 \cdot \Bernoulli(\frac{1}{2})-1$ and consequently $D_k - 
D_0 \sim 2 \cdot \Bin(k,\frac{1}{2})-k$. 
 As in the \mupoea, the initial $D_0$ satisfies $\prob{|D_0| < 
\mu^{\frac{1}{2}+\epsilon}}=\exp\left( -\Omega(\mu^{2\epsilon})\right)$. 
Therefore, for any relevant step $k$, 
	\begin{align*}
&\prob{|D_k|=\mu \quad \big| \quad |D_0| < 
\mu^{\frac{1}{2}+\epsilon}}\leq \prob{|D_k|-|D_0|\geq 
\mu-\mu^{\frac{1}{2}+\epsilon}}\leq \prob{D_k-D_0\geq \mu-\mu^{\frac{1}{2}+\epsilon}}\\
&= \prob{2\cdot X -k\geq 
\mu-\mu^{\frac{1}{2}+\epsilon}}= \prob{ X \geq \frac{\mu-\mu^{\frac{1}{2}+\epsilon}+k}{2}}= \prob{ X 
\geq 
\frac{\mu-\mu^{\frac{1}{2}+\epsilon}}{2}+E[X]}\\
&\leq\exp\left(-\frac{2\cdot\left(\mu-\mu^{\frac{1}{2}+\epsilon}\right)^2}
{\frac{k}{2 
}}\right)=\exp\left(-\Omega\left(\frac{\left(n^{1+\epsilon}\right)^2}{n^{
1+\epsilon}} \right)\right)=\exp\left(-\Omega\left(n^{\epsilon}\right)\right)
	\end{align*} 
		where $X\sim \Bin(k,\frac{1}{2})$ and in the last step we  used 
the observation that $k\leq \mu\cdot n$ and the assumption that 
$\mu=\Omega(n^{1+\epsilon})$. Using the above bound and the union bound over 
$k\in[\mathcal{K}]$ we get our final result:
%
\[
\prob{\forall k \in [\mathcal{K}] \quad |D_k|\neq \mu 
\quad \big| \quad |D_0|  < \mu^{\frac{1}{2}+\epsilon}}\geq 1- \mu \cdot n  
\cdot \exp\left(-\Omega\left(n^{\epsilon}\right)\right)= 1- 
\exp\left(-\Omega\left(n^{\epsilon}\right)\right)
\]
Thus, with probability $1- 
\exp\left(-\Omega\left(n^{\epsilon}\right)\right)$,  $|D_k|<\mu$ for all $k$ and 
the population is not taken over by any of the branches before all the solutions 
reach the optimal fitness.
			\end{proof}

\section{Not Too High, Nor Too Low Pressure: Experimental Supplements}\label{sec:mid}
It is easy to see that the Inverse Elitist selection operator considered in the previous section
may struggle to identify the global optimum if the worst individual is trapped on a low quality local optimum. In such a case it will be selected until it escapes which may require prohibitive time.

In this section we present an experimental analysis of the \mupoea using the selection operators considered in the previous sections (i.e., too high and extremely low selective pressure) together with a range of intermediate selective pressures that are higher than Inverse Elitist but still considerably lower than Uniform selection. To achieve such pressures we implement {\it Inverse Tournament selection} as defined in Section 2 with small tournament sizes\peter{\footnote{Naturally the same selection probabilities can be achieved with rank selection.}}. 
Our aim is to show that the latter operator is preferable to both Uniform and Inversely Elitist selection because it has low premature convergence for \twomax and it can effectively escape from local optima as experiments for \truncatedtwomax will show.  

The results presented in this section are based on performing 1000 independent
runs for each combination of problem size and algorithm parameters, and
observing whether the ($\mu$+1)~EA is able to construct both optima during the
search process. The presented figures estimate the probability that the
algorithms find both optima by the proportion of the independent runs which do
so. Additionally, a 95\% confidence interval is displayed as a shaded region
around the estimated probabilities.

\begin{figure}[t]
	\centering
\centering
\begin{tikzpicture}
\begin{semilogxaxis}[
width=0.75\linewidth,
height=0.40\linewidth,
legend style={nodes={scale=0.8, transform shape}, at={(1.02,0.5)}, anchor=west},
legend columns=1,
xlabel={Population size ($\mu$)},
ylabel={P(both peaks are found)},
ymin = -0.05, ymax=1.05,
try min ticks=5, log ticks with fixed point,
mark options={solid},mark repeat=5,
]
\plotwitherr{cyan}{mark=o}{mu}{UNIF/LOWM/RLS/n100}{1000}{pBoth-basic.tsv}
\plotwitherx{cyan}{mark=o}{mu}{UNIF/LOWM/SBM/n100}{1000}{pBoth-basic.tsv}
\plotwitherr{cyan}{mark=triangle}{mu}{UNIF/LOWM/RLS/n200}{1000}{pBoth-basic.tsv}
\plotwitherx{cyan}{mark=triangle}{mu}{UNIF/LOWM/SBM/n200}{1000}{pBoth-basic.tsv}
\plotwitherr{cyan}{mark=square}{mu}{UNIF/LOWM/RLS/n500}{1000}{pBoth-basic.tsv}
\plotwitherx{cyan}{mark=square}{mu}{UNIF/LOWM/SBM/n500}{1000}{pBoth-basic.tsv}
\plotwitherr{orange}{mark=*,mark size=1.5pt}{mu}{UNIF/UNIF/RLS/n100}{1000}{pBoth-basic.tsv}
\plotwitherx{orange}{mark=*,mark size=1.5pt}{mu}{UNIF/UNIF/SBM/n100}{1000}{pBoth-basic.tsv}
\plotwitherr{orange}{mark=triangle*,mark size=1.5pt}{mu}{UNIF/UNIF/RLS/n200}{1000}{pBoth-basic.tsv}
\plotwitherx{orange}{mark=triangle*,mark size=1.5pt}{mu}{UNIF/UNIF/SBM/n200}{1000}{pBoth-basic.tsv}
\plotwitherr{orange}{mark=square*,mark size=1.5pt}{mu}{UNIF/UNIF/RLS/n500}{1000}{pBoth-basic.tsv}
\plotwitherx{orange}{mark=square*,mark size=1.5pt}{mu}{UNIF/UNIF/SBM/n500}{1000}{pBoth-basic.tsv}

\legend{Inverse Elitist $n=100$,Inverse Elitist $n=200$,Inverse Elitist $n=500$, Uniform $n=100$, Uniform $n=200$, Uniform $n=500$}
\end{semilogxaxis}
\end{tikzpicture}
	\caption{Proportion of independent runs finding both peaks on \twomax for various combinations of parent selection
	operators, mutation operators (solid: RLS$_1$, dotted: standard bit mutation), and problem sizes.}
	\label{fig:experiments-inverse-and-uniform}
\end{figure}
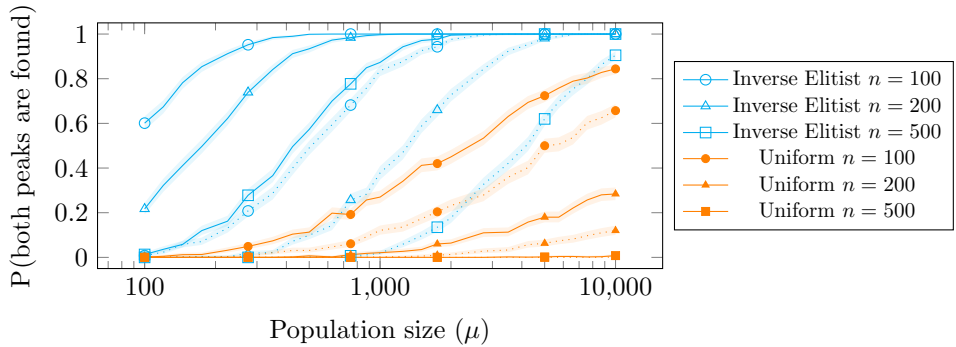

We start by examining the success probability of the analysed selection
operators on \twomax for small problem sizes.
Figure~\ref{fig:experiments-inverse-and-uniform} compares the performance of
Inverse Elitist and Uniform parent selection operators for various population
and problem sizes. As expected, Uniform parent selection requires a larger
population size than Inverse Elitist parent selection in order to achieve the
same probability of finding both peaks; for the considered problem sizes, $\mu
= n^{1.5}$ appears sufficient for the combination of Inverse Elitist parent
selection and RLS mutation to achieve a high probability of discovering both
peaks, while SBM mutation requires a slightly larger population size, and
Uniform parent selection requires a significantly larger population size.

\begin{figure}[t]
	\centering
\centering
\begin{tikzpicture}
\begin{semilogxaxis}[
width=0.75\linewidth,
height=0.40\linewidth,
legend style={nodes={scale=0.8, transform shape}, at={(1.02,0.5)}, anchor=west}, 
legend columns=1,
xlabel={Population size ($\mu$)},
ylabel={P(both peaks are found)},
ymin = -0.05, ymax=1.05,
try min ticks=5, log ticks with fixed point,
mark options={solid},mark repeat=5,
]
\plotwitherr{cyan}{thick,mark=*, mark size=1.5pt}{mu}{UNIF/LOWM/RLS/n200}{1000}{pBoth-basic.tsv}
\plotwitherx{cyan}{thick,mark=*, mark size=1.5pt}{mu}{UNIF/LOWM/SBM/n200}{1000}{pBoth-basic.tsv}
\plotwitherr{blue}{mark=triangle}{mu}{UNIF/TOURL/RLS/k2/n200}{1000}{pBoth-tournament.tsv}
\plotwitherx{blue}{mark=triangle}{mu}{UNIF/TOURL/SBM/k2/n200}{1000}{pBoth-tournament.tsv}
\plotwitherr{orange}{mark=square}{mu}{UNIF/TOURL/RLS/k3/n200}{1000}{pBoth-tournament.tsv};
\plotwitherx{orange}{mark=square}{mu}{UNIF/TOURL/SBM/k3/n200}{1000}{pBoth-tournament.tsv};
\plotwitherr{green!60!black}{mark=o}{mu}{UNIF/TOURL/RLS/k5/n200}{1000}{pBoth-tournament.tsv}
\plotwitherx{green!60!black}{mark=o}{mu}{UNIF/TOURL/SBM/k5/n200}{1000}{pBoth-tournament.tsv}
\plotwitherr{purple}{thick,mark=triangle*, mark size=1pt}{mu}{UNIF/UNIF/RLS/n200}{1000}{pBoth-basic.tsv}
\plotwitherx{purple}{thick,mark=triangle*, mark size=1pt}{mu}{UNIF/UNIF/SBM/n200}{1000}{pBoth-basic.tsv}
\plotwitherr{blue!60!black}{mark=square*, mark size=1pt}{mu}{UNIF/TOURH/RLS/k2/n200}{1000}{pBoth-tournament.tsv}
\plotwitherx{blue!60!black}{mark=square*, mark size=1pt}{mu}{UNIF/TOURH/SBM/k2/n200}{1000}{pBoth-tournament.tsv}

\legend{Inverse Elitist,Inv. 2-Tournament, Inv. 3-Tournament, Inv. 5-Tournament,
Uniform, 2-Tournament}
\end{semilogxaxis}
\end{tikzpicture}
	\caption{Proportion of independent runs finding both peaks on \twomax ($n=200$) for various combinations of parent selection
	and mutation operators (solid: RLS$_1$, dotted: standard bit mutation).}
	\label{fig:experiments-inverse-tournament}
\end{figure}
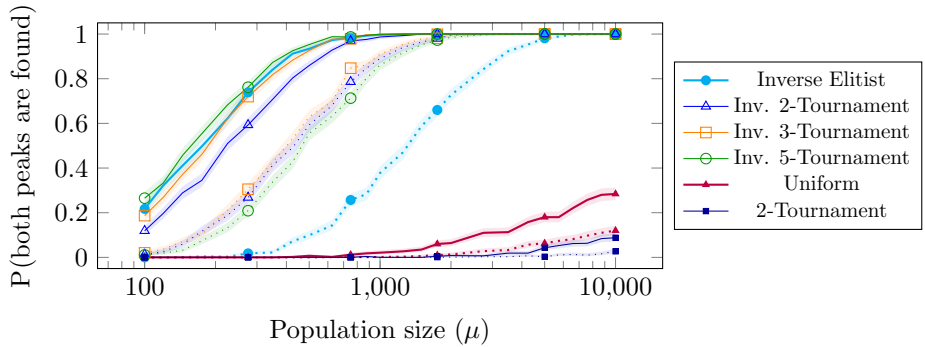

Figure~\ref{fig:experiments-inverse-tournament} compares the performance of
Tournament and Inverse Tournament parent selection operators with that of
Inverse Elitist and Uniform parent selection operators for a fixed problem size
and varying population sizes. Generally, modest Inverse Tournament sizes (such
as $k=3$) are sufficient to achieve similar-or-better probability of
discovering both optima compared to Inverse Elitist parent selection when using
RLS mutation, and can even outperform Inverse Elitist parent selection when
using standard bit mutation. As expected, the higher selective pressure of the
traditional Tournament selection results in faster takeover by a single
branch, and hence a lower probability of finding both optima.

\begin{figure}[t]
	\centering
\centering
\begin{tikzpicture}
\begin{semilogxaxis}[
width=0.75\linewidth,
height=0.40\linewidth,
legend style={nodes={scale=0.8, transform shape}, at={(1.02,0.5)}, anchor=west},
legend columns=1,
xlabel={Population size ($\mu$)},
ylabel={P(both peaks are found)},
ymin = -0.05, ymax=1.05,
try min ticks=5, log ticks with fixed point,
mark options={solid},mark repeat=5,
]
\plotwitherr{cyan}{thick,mark=*,mark size=1.5pt}{mu}{UNIF/LOWM/RLS/n200}{1000}{pBoth-truncated.tsv}
\plotwitherx{cyan}{thick,mark=*,mark size=1.5pt}{mu}{UNIF/LOWM/SBM/n200}{1000}{pBoth-truncated.tsv}
\plotwitherr{blue}{mark=triangle}{mu}{UNIF/TOURL/RLS/k2/n200}{1000}{pBoth-truncated.tsv}
\plotwitherx{blue}{mark=triangle}{mu}{UNIF/TOURL/SBM/k2/n200}{1000}{pBoth-truncated.tsv}
\plotwitherr{orange}{mark=square}{mu}{UNIF/TOURL/RLS/k3/n200}{1000}{pBoth-truncated.tsv}
\plotwitherx{orange}{mark=square}{mu}{UNIF/TOURL/SBM/k3/n200}{1000}{pBoth-truncated.tsv}
\plotwitherr{green!60!black}{mark=o}{mu}{UNIF/TOURL/RLS/k5/n200}{1000}{pBoth-truncated.tsv}
\plotwitherx{green!60!black}{mark=o}{mu}{UNIF/TOURL/SBM/k5/n200}{1000}{pBoth-truncated.tsv}
\plotwitherr{purple}{thick,mark=triangle*,mark size=1pt}{mu}{UNIF/UNIF/RLS/n200}{1000}{pBoth-truncated.tsv}
\plotwitherx{purple}{thick,mark=triangle*,mark size=1pt}{mu}{UNIF/UNIF/SBM/n200}{1000}{pBoth-truncated.tsv}
\plotwitherr{blue!60!black}{mark=square*,mark size=1pt}{mu}{UNIF/TOURH/RLS/k2/n200}{1000}{pBoth-truncated.tsv}
\plotwitherx{blue!60!black}{mark=square*,mark size=1pt}{mu}{UNIF/TOURH/SBM/k2/n200}{1000}{pBoth-truncated.tsv}

\legend{Inverse Elitist,Inv. 2-Tournament, Inv. 3-Tournament, Inv. 5-Tournament,
Uniform, 2-Tournament}
\end{semilogxaxis}
\end{tikzpicture}
	\caption{Proportion of independent runs finding both peaks on \truncatedtwomax ($n=200$, $k=80$) for various combinations of parent selection
	operators and mutation operators (solid: RLS$_1$, dotted: standard bit mutation).}
	\label{fig:experiments-truncated-twomax}
\end{figure}
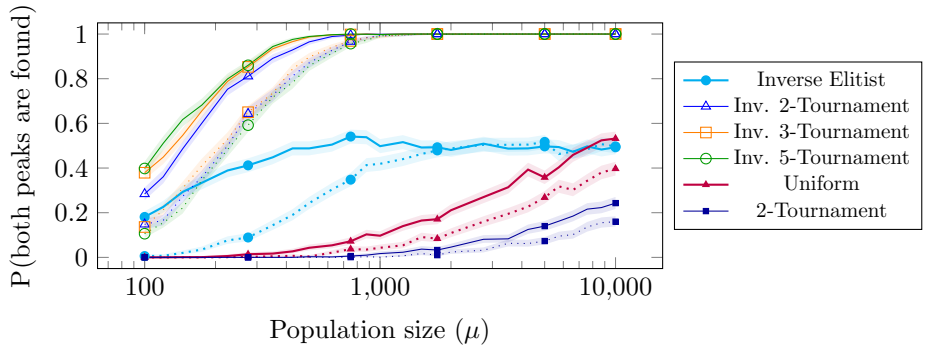

For \truncatedtwomax, where the left branch terminates at an earlier local
optimum, Figure~\ref{fig:experiments-truncated-twomax} shows that with
sufficiently high population sizes, Inverse Elitist parent selection is only
able to find both peaks with probability at most $1/2$, while the other parent
selection mechanisms are able to achieve higher probabilities of finding both
peaks compared to the non-truncated setting of
Figure~\ref{fig:experiments-inverse-tournament}. Additionally, Inverse
Tournament Selection suffers a runtime penalty in this setting: when the time
comes to remove from the population the local optimum on the truncated branch, 
larger tournament sizes result in higher waiting times to select a
suitable parent to mutate. This is particularly pronounced when the population
size only provides a moderate (rather than high) probability of finding both
peaks. If the population is almost entirely on the truncated branch when that
local optimum is the lowest-fitness individual in the population, selecting a
parent from the other branch could require up to $O(\mu^{s})$ iterations, where
$s$ is the employed Inverse Tournament size, and would need to be repeated many
times to restore a reasonable probability of an iteration making progress
toward the global optimum. We conclude that small (i.e., constant) inverse tournament 
sizes should be used in practice.

For the problems considered up to now, the use of single trajectory algorithms coupled with restart strategies suffice for efficient optimisation.
In the following sections we consider problems from the literature where restart strategies are ineffective.

\section{Where Populations are Essential Because Restart Strategies Fail} \label{sec:norestarts}
In the previous sections we have shown that decreasing the selective pressure below uniform, increases the exploration capabilities of the \mupoea. Essentially lower selective pressure increases the takeover time, effectively allowing the population to explore the two slopes of \twomax. However, when multiple slopes can be identified with high probability at initialisation as for \twomax, then single trajectory algorithms with appropriate restart strategies can be just as, or even more, effective.

In this section we consider two benchmark function classes from the literature originally designed 
to highlight circumstances where populations are essential and single trajectory algorithms fail w.o.p.
The two function classes have different characteristics. The first 
one,  \textsc{SufSamp}, was introduced to show situations when offspring populations outperform 
single trajectory algorithms \cite{JansenSufSamp}. The function consists of a main steep path 
leading to the global optimum with many branching points located at regular intervals. At each such 
point a branch leading towards a different local optimum may be followed. While single 
trajectory algorithms are tricked into following one of the branches, algorithms using large enough 
offspring populations can identify the steeper slope by sampling many solutions around each 
branching point.  We will refer to this function using a more descriptive name: \sufsamp.
The other function class, \textsc{TwoGradients}, is inspired by one designed to highlight when the 
use of parent populations is essential~\cite{WittFamily}. It consists of two different paths: an 
easier one for hillclimbers leading to local optima and a slower one which leads to the global 
optimum. While single trajectory algorithms quickly optimise the easier path, the slower progress of 
parent populations on the easy path allows the \mupoea to follow the not-so-steep gradient to the 
global optimum.
Since for both function classes, 
single trajectory algorithms are attracted w.o.p. 
to the local optima, even restarts do not allow them to identify the global optimum in polynomial time. Apart from proving that the \mupoea with the inverse selection operators optimises the problems efficiently, in this section we will also show, either through rigorous proofs or by experimental evidence, that it outperforms the same algorithm using uniform selection: the inverse selection algorithms can identify all the optima with high probability, hence optimise the functions efficiently independently of which local optimum is designated to be the global one. On the other hand, single trajectory algorithms and the uniform selection \mupoea can be trapped w.o.p. 
according to which optimum is the global one and even restarts cannot help them to optimise the functions efficiently.


\subsection{A Function Class with Many Slopes}


In this section we aim to highlight an important exploration characteristic of the \mupoea with inverse tournament selection that single trajectory algorithms using restarts or population-based EAs with deterministic crowding do not have. In the latter algorithms each individual explores the search space independently and no communication occurs between individuals that explore different trajectories.  
Differently, 
inverse selection does not completely isolate the individuals in the population from each other. If individuals are located on different slopes,
it is more likely that slopes of lower fitness are explored first. Once the local optima of the slopes are identified, individuals on slopes of higher fitness will eventually take over and the search may proceed in other unexplored directions.
This property allows an EA with inverse selection to exhaust the improvement potential of many of the gradients it encounters with good probability.  
In particular, it allows the algorithm to explore gradients which more greedy search strategies ({\color{black}e.g.}, using offspring populations or parent populations with low take-over times) are likely to ignore, thus identify optima that popular algorithms will struggle to find.
To demonstrate this behaviour, we consider 
the \sufsamp function which has many slopes leading to different local optima. It was originally 
introduced  to highlight problem characteristics where offspring populations are useful 
\cite{JansenSufSamp}. 
It is defined as follows: 

\begin{definition}
For $n\in \mathbb{N}$ we define $k:= \lfloor \sqrt{n} \rfloor$. We use $|x|=\textsc{OneMax}(x)$ and 
define the function $g:\{0,1\}^n \rightarrow \mathbb{N}$ for all $x\in \{0,1\}^n$ 
by
\label{def-sufsampcore}
\[
g(x):=
\begin{cases}
(i+3)n +|x| &\text{if } (x=0^{n-i}1^{i} \text{ with } 0\leq i \leq n ) \text{ or } \\
&(x=y0^{n-i-k}1^{i} \text{ with } \\&i \in \{k, 2k, \ldots, (k-2)k\}, \\&y\in \{0,1\}^k )
\\
0 & \text{otherwise,}
\end{cases}
\]
for all $x=x_1x_2\cdots x_n \in \{0,1\}^n$  with $x':= x_1x_2\cdots x_{m^{'}}\in \{0,1\}^{m^{'}}$  and  $x':= x_1x_2\cdots x_{m^{''}}\in \{0,1\}^{m^{''}}$.
\end{definition}

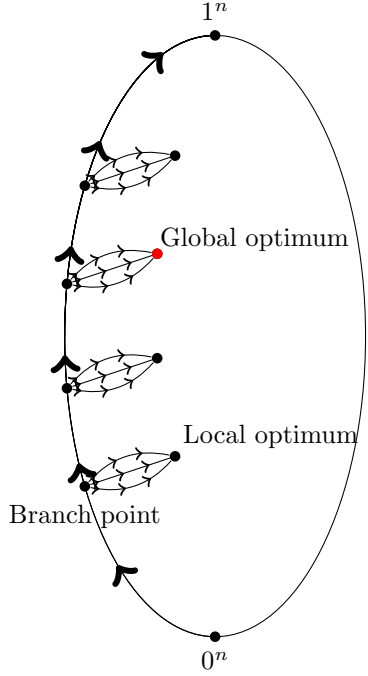
\begin{figure} 
\caption{Visual representation of \sufsamp. Arrows indicate the direction of increasing fitness while multiple edges represent different ways  the branch points can be improved towards the local optima.}
\begin{tikzpicture}

\draw (0,0) node(al1) [circle, fill, inner sep=0pt,,minimum size=4pt, label=above:{$1^{n}$}] {}  arc (90:540):20mm and 40mm) ; 

\draw (al1) arc (90:270:20mm and 40mm)
	[postaction={
							decorate, decoration={
																			markings, mark=between positions 0.1 and 0.9 step 0.15 with thick {\arrow[line width=1mm]{<}}
																		}
						}
        ] node(al0)[circle, fill, inner sep=0pt,,minimum size=4pt, label=below:{$0^n$} ] {};

\foreach \t [count=\i] in {210,190,170,150}
{\draw (al1)  arc (90:\t:20mm and 40mm) node(b\i)[circle, fill, inner sep=0pt,,minimum size=4pt, label=below:{}] {};};

\foreach \t [count=\i] in {1,2,3,4}
{\draw [postaction={
							decorate, decoration={
																			markings, mark=between positions 0.1 and 0.9 step 0.3 with thick {\arrow[line width=0.3mm]{>}}
																		}
						}](b\i) -- ($(b\i)+(1.2,0.4)$) node[circle, fill, inner sep=0pt,minimum size=4pt,label={[xshift=1.0cm, yshift=0.3cm]}](l\i) {};

\draw [bend right,postaction={
							decorate, decoration={
																			markings, mark=between positions 0.1 and 0.9 step 0.3 with thick {\arrow[line width=0.3mm]{>}}
																		}
						}] (b\i) to (l\i);
\draw [bend left,postaction={
							decorate, decoration={
																			markings, mark=between positions 0.1 and 0.9 step 0.3 with thick {\arrow[line width=0.3mm]{>}}
																		}
						}] (b\i) to (l\i);						
						
 };

\draw (b1) node[label=below:{Branch point}] {};
\draw (l1) node[label={[label distance=-2mm]45:{Local optimum}}] {};
 \draw (l3) node[circle, fill, red, inner sep=0pt,,minimum size=4pt, label={[label distance=-2mm]45:{Global optimum}} ] {};

				\end{tikzpicture}
						\centering
						\label{fig:sufSamp}
				\end{figure}

We will use the function $g$ to give \sufsamp its characteristic structure. It is essentially a 
\text{Ridge} function (where for $i\in\{0,\ldots,n\}$ \emph{Ridge solutions} of the form 
$0^{n-i}1^i$ are assigned fitness values which monotonically increase with $i$) with 
special \emph{branch points} 
of the form $0^{n-j\cdot k}1^{j\cdot k}$ for any $j\in\{1,\ldots, k-1\}$ on its path at every 
$k=\lfloor \sqrt{n} \rfloor$ 
additional trailing ones. These branch points can either be improved by increasing the \onemax 
value of its prefix (thus, obtaining an \emph{off-branch} solution of the form 
$y0^{n-(j+2)k}1^{(j+1)\cdot k}$ where $y\neq 0^k$ ) or by flipping the first zero bit 
before its trailing ones. While improving the \onemax part of the function is asymptotically more 
likely,  flipping the \text{Ridge} bit increases the function value more than it is possible by 
maximising the \onemax part in the prefix. Moreover, flipping the particular \text{Ridge} bit allows 
the algorithm to follow the gradient to the next branch point. In its current form $g$ does not 
provide any guidance to the algorithm for finding the first point on the \textsc{Ridge}.  We add 
this extra gradient to obtain the \sufsamp function as defined in \cite{JansenSufSamp}.
\begin{definition}
For $n\in \mathbb{N}$ we define $m^{'}:= \lfloor n/2 \rfloor$ and $m^{''}:= \lceil n/2 \rceil$, and 
$\sufsamp:\{0,1\}^n \rightarrow \mathbb{N}$
\label{def-sufsamp}
\[
\sufsamp(x):=
\begin{cases}
n-\textsc{OM}(x^{''})&\text{if } x'\neq 0^{m^{'}} \wedge x'\neq 0^{m^{''}}  \\
2n-\textsc{OM}(x^{'})&
\text{if } x' \neq 0^{m^{'}} \wedge x'= 0^{m^{''}}
\\
g(x^{''})&\text{if } x'= 0^{m^{'}} 
\end{cases}\]
for all $x=x_1x_2\cdots x_n \in \{0,1\}^n$  ~with~  $x':= x_1x_2\cdots x_{m^{'}}\in \{0,1\}^{m^{'}}$  ~and~   $x':= x_1x_2\cdots x_{m^{''}}\in \{0,1\}^{m^{'}}$.
\end{definition}
Since it is more likely to flip a \onemax bit rather than the \text{Ridge} bit at each branching point, single trajectory individuals, such as RLS and the \ooea, get stuck on a local optimum w.o.p., 
hence restarts cannot help either. On the other hand, an offspring population algorithm such as the (1+$\lambda$)~EA with sufficiently large $\lambda$, samples  %
the \text{Ridge} bit at each branching point w.o.p., hence efficiently follows the slope towards the global optimum.

We generalise the function class to make it considerably harder for evolutionary algorithms by including $k-1$ different functions, each one differing by the position of the global optimum.
In particular, we introduce a parameter $j\in\{1, \ldots, 
k-1 \}$ and for $\sufsamp_j$ the 
optimal fitness value is assigned to the local optimum with $j\cdot k$ trailing ones. 
\textcolor{black}{The function class is illustrated in Fig.~\ref{fig:sufSamp}.}

\begin{definition}
For $k\in \mathbb{N}$, $m'=m''=k^2$, $n=m'+m''$ and $j\in\{1, \ldots, k-1 \}$, we define 
$\sufsamp_j(x):\{0,1\}^n \rightarrow \mathbb{N}$
\label{def-sufsampmod}
\[
\sufsamp_{j}(x):=
\begin{cases}
n^3 &\text{if } x=0^{m'}1^{k}0^{m'-(j+1)\cdot k}1^{j\cdot k}  \\
\sufsamp(x)&
\text{otherwise } 
\end{cases}\]
\end{definition}

As a result we obtain a function class that none of the traditional EAs, the restart based 
approaches or the inverse elitist selection \mupoea can optimise efficiently. The 
$(1+\lambda)$~EA with an offspring population of size $\Omega(n\log{n})$ ignores the milder gradients 
and with probability arbitrarily close to $1$ finds a local optimum (\emph{i.e.} the global optimum of the original \sufsamp)
 \cite{JansenSufSamp}.  When the offspring population 
size is slightly increased to $\Omega(n^{1+\epsilon})$ for any arbitrarily small constant 
$\epsilon$, then the probability that  $(1+\lambda)$~EA follows any of the milder gradients 
becomes exponentially small and the global optimum of $\sufsamp_{j}$ cannot be found in 
polynomial time for any $j\neq k-1$ w.o.p. On the other hand, EAs with small 
populations (whether it is a single run or a restart based implementation) have a high probability 
of following any of the mild gradients they encounter and thus have an exponentially small 
probability of finding the optima  that require a super-constant number of mild gradients to be 
ignored.  
\textcolor{black}{If the population is large enough to ignore the mild gradients with high probability, then the \mupoea is very likely to end up on the original global optimum at the end of the Ridge.}
Noting that $\sufsamp_{k-1}$ is identical to the original 
function, 
we first extend the negative result presented in 
\cite{JansenSufSamp} {\color{black} with a proof sketch that adapts it }
to $\sufsamp_{j}$.
\begin{theorem}
 \label{thm-sufsamp}
 The probability that the \ooea optimises the function $\sufsamp_{j}$ for any $j>k/2$
within $n^{O(1)}$ function evaluations is bounded above by $2^{-\Omega(\sqrt{n})}$.
\end{theorem}

We follow the proof of Theorem 7 in \cite{JansenSufSamp}  
to establish that at every branch point, the \ooea samples a locally optimal solution with $k$ 
leading ones before it samples a new ridge solution with probability $1-O(1/\sqrt{n})$. Since 
$j>k/2$ implies that the best fitness value of $\sufsamp_{j}$ is assigned to a local optimum with at 
least $m''/2$ trailing ones, we multiply 
the failure probabilities of only the first $\sqrt{n}/4$ branch points (rather than all of them as 
in \cite{JansenSufSamp}) and obtain the same asymptotic failure probability 
$2^{-\Omega(\sqrt{n}\log{n})}$. This allows us to report the same overall failure probability as 
in \cite{JansenSufSamp} which stems from the probability that the first ridge point sampled by the 
\ooea has more than $2k$ trailing ones.

The exponentially small probability given in Theorem~\ref{thm-sufsamp}  implies that at least an 
exponentially large number of restarts of the \ooea are necessary to observe at least one run that 
identifies the global optimum in polynomial time. 
Thus, any restart strategy using the \ooea as the base algorithm, whether the runs are in parallel or sequential, cannot solve \sufsamp in polynomial time. Hence, also a \mupoea using deterministic crowding as diversity mechanism fails w.o.p. 
We now present the main result of this subsection which determines that the \mupoea with inverse binary tournament selection identifies all the optima of \sufsamp efficiently, hence also the global one wherever it is designated to be. 
\begin{theorem}
 \label{thm-sufsampNeg}
 The probability that the \mupoea with polynomial population size $\mu=\omega(n^{3})$ and inverse 
binary tournament {\color{black}selection}, optimises the function $\sufsamp_{j}$ for any $j \in \{1,\ldots, k-1\}$ in 
$O(\mu^2 \sqrt{n})$ fitness function evaluations is  at least $1-2^{-\Omega(\sqrt{n})}$.
\end{theorem}
\begin{proof}
 
 We will initially follow some steps from the proof of Theorem~\ref{thm-sufsampNeg} in 
\cite{JansenSufSamp}. A randomly initialised bit-string satisfies $x'\neq 0^{m'}$ and $m''/3 < 
\onemax(x'')< 2m''/3$ with probability $1-2^{-\Omega(\sqrt{n})}$. The probability that all 
individuals in the initial population satisfy the same conditions has the same asymptotic order of 
$1-2^{-\Omega(\sqrt{n})}$ due to the union bound. For any $j\in \{1,\ldots,k-2\}$, let $b_j$ to 
denote the $j$th branch point $0^{m'}0^{m''-j\cdot k}1^{j\cdot k}$.
Given that the 
initial population satisfies the 
conditions and no bit-strings with higher fitness are found before, the expected time until the 
first 
branch point $b_1:=0^{m'}0^{m''-k}1^{k}$ 
is sampled for the first time is less than the sum of the expected 
$O(\mu n \log{n})$ evaluations for finding the  $0^{n}$ bit-string and expected $O(\mu n^{3/2})$ 
fitness function evaluations for improving the function $g$, $2k=\Theta(\sqrt{n})$ times for each 
individual. Consequently, the probability that the actual time until $b_1$ is sampled is $O(\mu n^2 
\log{n})$ is $1-2^{-\Omega(\sqrt{n})}$.  We will next bound the probability of sampling a solution 
$y$ with $\sufsamp(y)>\sufsamp(b_1)$ before sampling $b_1$. Considering levels of strings with a 
given 
number of zero-bits, due to symmetry, each bit-string on a level have the same probability of being 
sampled before the $0^n$ bitstring, the first Ridge solution, is sampled. Between levels $m''/3$ and $m''-2k$, there are 
$2^{\Omega{(k)}}$ bit-strings in total and the probability of sampling a Ridge solution on these 
levels in the first $O(\mu n^2 \log{n})$ steps is at most $2^{-\Omega(\sqrt{n})}$ by the union 
bound. Therefore, with probability $1-2^{-\Omega(\sqrt{n})}$ our population will consist only of 
Ridge solutions with worse fitness than $b_1$ at some point during optimisation.

 Next, we will consider a phase where, 
for some 
fixed $j\in \{0,\ldots, k-2\}$, all initial individuals in the population are 
Ridge solutions  with more than $j\cdot 
k$ and less than $(j+1)\cdot k$ trailing ones. 
We will now bound the probability that the off-branch 
solutions take over the population before the Ridge solution $0^{m'}0^{m''-(j+1)\cdot 
k-1}1^{(j+1)\cdot 
k+1}$ 
is sampled for the first 
time. Since this Ridge solution has higher fitness value than all off-branch solutions 
with smaller number of trailing ones, once created, this solution (which we will refer as a \emph
{safe} solution) cannot be removed by adding more off-branch solutions to the population.

It is exponentially unlikely that an off-branch solution of the form 
$0^{m'}y0^{m''-(j+2)k}1^{(j+1)\cdot k}$ 
is created from a parent with less than $j\cdot k$ trailing ones. In particular, this 
event requires at least $\Omega(\sqrt{n})$ specific bits to be flipped which occurs with 
probability at most 
$n^{-\Omega{(\sqrt{n})}}$. We will condition on this event never happening and absorb this 
probability into our failure probability of $2^{-\Omega(\sqrt{n})}$ in our claim.  Thus, 
we can now categorise the newly 
created off-branch solutions into two types according to their parents: (1) whose parents are 
Ridge solutions with at least $j\cdot k$ trailing ones and (2) whose parents are also off-branch 
solutions with $(j+1)\cdot k$ trailing ones. From a Ridge solution, the conditional probability of 
creating a safe solution given that either an off-branch solution or a safe solution has been 
created is $\Theta(1/\sqrt{n})$ since only 
$\Theta(\sqrt{n})$ bits contribute to the additive $|x|$ term in the definition of~$g$.
Thus, with probability $1-2^{-\Omega(\sqrt{n})}$ we obtain a safe solution before obtaining $n/2$ 
type-1 off-branch solutions.

Now, we turn to the off-branch solutions whose parents are also off-branch solutions (type-2). Given that 
there are $k$ off-branch individuals in the population and no safe solutions, the probability of 
selecting an off-branch solution as parent is at most $(k/\mu)^2$ since the off-branch solution have higher 
fitness value than the rest of the population and inverse tournament selection returns the candidate 
with the lowest fitness.

 Pessimistically assuming that the off-branch sub-population initially has $n/2$ individuals mutated 
from Ridge solutions, we will bound the time until the total number of off-branch solution increases from $n/2$ to $n$, \emph{i.e.}, we 
create $n/2$ type-2 individuals. We will 
pessimistically assume that the probability of picking an off-branch individual is $(n/\mu)^2$, 
thus, using Chernoff bounds on geometric variables \cite{DoerrBookChapter}, it takes at least, 
$\frac{n \cdot (\mu/n)^2}{3} = \Omega(\mu^2/n)$ iterations with probability 
$1-2^{-\Omega(\mu^2/n)}$ 
before there are 
$n$ off-branch individuals in the population. On the other hand in $O(\mu n^{3/2})$ iterations, the 
whole population consists of individuals with at least as fit as the branch point $b_{j+1}$ since at most 
$O(\sqrt{n})$ Ridge points has to be traversed for each solution, where the improvement probability 
is exactly $1/n$. W.o.p. 
this process takes less than  $O(\mu n^2)\subset o( \mu^2/n)$ considering $\mu=\omega(n^3)$. Since 
we have already established that only $n$ individuals can be on the off-branch, at least a 
$1-o(1)$ fraction of the population is at the branch point $b_{j+1}$. Thus, a branch point is selected as parent with constant probability and it takes at most $O(n^{2})$ iterations before a safe solution is created with  probability at least $1-2^{-\Omega(\sqrt{n})}$.

Once the safe solution is created we apply the same line of argument in reverse to now show that 
off-branch individuals optimise the \onemax prefix before the safe solution takes over the population. 
However, this line of argumentation is relatively easier since the number of off-branch individuals 
are initially asymptotically larger than the number of safe solutions. Given that the population 
consists only of off-branch and safe solutions, in $O(\mu n^2 \log {n} )$ iterations the local 
optima $0^{m'}1^k0^{m''-(j+2)\cdot k}1^{(j+1)\cdot k}$ is found w.o.p. 
since a $1-o(1)$ fraction of the population is optimising the \onemax landscape. Increasing the number of safe solutions from $\sqrt{n}/2$ to $\sqrt{n}$ takes at least $\frac{\sqrt{n} \cdot (\mu/\sqrt{n})^2}{2}$ iterations w.o.p. 
due to Chernoff bounds on geometric variables. Since the takeover time is asymptotically larger than 
the \onemax optimisation time, we can conclude that at every branch point the local optima is 
sampled and the safe solution is added to the population with probability at least 
$1-2^{\Omega(\sqrt{n})}$. After the local optima is found, the off-branch solution can no longer 
improve and the progress is stagnated until the safe solution takes over the population in at most 
$O(\mu^2)$ iterations. Since this take over has to happen at all $\sqrt{n}/2$ branch points, the 
total 
takeover time is at most $O(\mu^2 \sqrt{n})$ with probability $1-2^{-\Omega(\sqrt{n})}$ and this 
term dominates the runtime. 

 
\end{proof}

\subsection{A Function Class with Two Slopes}
In the previous subsection we considered a function with many different slopes and only one global optimum.
We showed that inverse binary selection makes the \mupoea efficient for the function by allowing it to explore all the slopes one at a time, while all other considered evolutionary algorithms require exponential time w.o.p.
In this subsection we study a function class with two different slopes that have to be explored in parallel. This function class was originally introduced by
Witt 
to illustrate problem characteristics where large parent populations are essential \cite{WittFamily}. More precisely, he defined 
a function called~$f$ where the \ooea  w.o.p. 
reaches a local optimum which needs exponential 
time to escape from, 
while the classical \muoneea with uniform parent selection and population size $\mu\ge n/\!\ln(en)$ 
finds the global optimum in expected polynomial time. 
Overall, according to which of the two optima is the  global one, parent populations may be either essential or detrimental w.o.p. (i.e., restarts do not change the outcome). 
Our aim in this subsection is to show that  by using inverse tournament selection, the \mupoea can efficiently identify both optima with high probability either in a single run (large tournaments) or via restarts (smaller tournaments).
%
We first prove that the function with the original global optimum
can also be optimized efficiently by the \muoneea with inverse parent selection (here inverse binary tournament selection) 
and a sufficiently large population. Afterwards we will show experimentally that both optima can be identified already for not too large population and tournament sizes i.e., the algorithm is preferable to the \ooea and the \mupoea with uniform selection.

For simplicity, we do not use the original definition~$f$ 
from \cite{WittFamily} here, but a version that still allows for an exponential performance difference 
using considerably shorter proofs. We remark that the simplified version does not allow a proof of polynomial time 
in expectation  but still a polynomial time bound that occurs w.o.p. 
 \ie, 
efficient optimisation w.o.p. 
Although we think that the \muoneea with 
inverse binary tournament selection and appropriate~$\mu$ indeed has expected polynomial runtime on the original function~$f$, 
we do not find the additional proof overhead  that would be necessary for such a result appropriate.

The function we study is adapted from \cite{WittFamily} and is a weighted sum of 
two independent optimisation problems defined on disjoint subspaces of $\{0,1\}^n$. On the last  
$\ell\coloneqq n^{1/3}$ bits, the well-known \textsc{LeadingOnes} problem has to be solved, more precisely 
for $x=(x_1,\dots,x_n)$ 
we define $\LSO(x)\coloneqq \sum_{i=1}^{\ell} \prod_{j=1}^{i} x_{n-\ell+j}$ as the number of 
\emph{leading suffix ones} value. On the first $m\coloneqq n-\ell$ bits, a \onemax problem has be be solved, hence 
we define $\PO(x)\coloneqq \sum_{i=1}^{m} x_i$ as the number of \emph{prefix ones}. The function 
is then defined by
\[
\sufpar(x) 
\coloneqq \begin{cases}
n^2 \LSO(x) + \PO(x) & \text{if $\PO(x)\le 2m/3$}\\
n^2 \ell - m - 1 + \PO(x) & \text{otherwise}.
\end{cases}
\]

The function $\sufpar$ has a global optimum at $\LSO(x)=\ell$ and $\PO(x)=2m/3$. To reach this, it is  typically required 
that the number of leading suffix ones is maximized before the number of prefix ones grows beyond $2m/3$. Any search point 
where $\PO(x)> 2m/3$ is better than any search point with non-optimal $\LSO$-value. Hence, if number of prefix 
ones is increased faster than the leading suffix ones and a PO-value of $2m/3$ is exceeded before optimising the suffix 
then the function leads to a local optimum at $\PO(x)=m$. This local optimum is a trap since at least 
$m/3$ prefix ones have to be flipped to reach the global optimum from there.

We now show (similarly to the original result from \cite{WittFamily}) that the \ooea w.o.p. 
fails to optimize $\sufpar$ efficiently.

\begin{theorem}
With probability at least $1-2^{-\Omega(n^{1/3})}$, the \ooea needs time $2^{\Omega(n\log n)}$ to 
optimize $\sufpar$.
\end{theorem}

\begin{proof}
The main idea is to show that w.o.p.
the \ooea reaches a so-called bad search point 
where $\PO(x)\ge 3m/4$ and $\LSO(x)<\ell$. From all bad search points, it is necessary to flip at least 
$m/12=\Omega(n)$ bits simultaneously to reach the global optimum. The probability of this is at most 
$1/(m/12)!=2^{-\Omega(n\log n)}$ and according to a union bound, 
with probability $1-2^{-\Omega(n\log n)}$ this does not happen in $2^{cn\log n}$ iterations if $c$ is chosen 
as a sufficiently small constant.

To show that a bad search point is reached w.o.p., 
we study the growth of 
the $\PO$-value carefully. First of all, by Chernoff bounds with probability $1-2^{-\Omega(n)}$, it
holds $\PO(x_0)\ge m/3$ for the initial search point~$x_0$. Moreover, 
$\prob{\LSO(x_0)\ge \ell/4} = 2^{-\ell/4}$ since this event requires to initialize the first $\ell/4$ suffix to all ones. 

 We assume $\PO(x_0)\ge m/3$ and $\LSO(x)\le \ell/4$ to happen, which altogether happens 
with probability $1-2^{-\Omega(n)} - 2^{-\Omega(\ell)} = 1-2^{-\Omega(n^{1/3})}$. Next, we note 
from the definition of $\sufpar$ that the $\PO$-value can only decrease in steps increasing the 
$\LSO$-value. Finally, we note that the number of leading suffix ones can increase by 
more than one when the leftmost zero in the suffix flips. More precisely, it increases by the number 
of so-called free-riders, the bits being random after the leftmost suffix zero. It is well known 
that the number of free-riders is stochastically bounded by a geometric distribution with parameter~$1/2$. 

We claim that it is 
sufficient to observe the following set of four events in a phase of length~$s\coloneqq 3em$ to create a bad search point, 
pessimistically assuming that no bad search is created before:
\begin{itemize}
\item There are at least $m/2$ $\PO$-increasing steps, 
\item there are at most $\ell/4$ $\LSO$-increasing steps, 
\item the total number of prefix bits flipped in at most $\ell/4$ $\LSO$-increasing steps is 
at most $\ell/2$, 
\item the total number of free-riders in at most $\ell/4$ $\LSO$-increasing steps is at most $\ell/3$.
\end{itemize}
If all these events together happen, the number of prefix ones increases  to 
at least $m/3 + m/2 - \ell/2 \ge 3m/4$ and the number of leading suffix ones to at most $\ell/4 + 
\ell/4 + \ell/3 < \ell$, which  corresponds to a bad search point. Hence, we only have to verify that 
each of the four events happens with probability at least $1-2^{-\Omega(n^{1/3})}$, from which the theorem 
follows via a union bound.

The first event happens with probability $1-2^{-\Omega(m)}$ according to Chernoff bounds since each of the $s=3em$ steps flips 
one of the least $m/4$ zero-bits with probability at least $(m/4)(1/n)(1-1/n)^{n-1} \ge 1/(5e)$ for $n$ large enough. The second 
event occurs with probability $1-2^{-\Omega(\ell)}$ by Chernoff bounds 
 since the probability of an $\LSO$-increasing step is $1/n$, resulting 
in an expected number of at most $3e$ such steps. The third event is also shown to occur with 
probability $1-2^{-\Omega(\ell)}$ via Chernoff bounds since each of the $m$ prefix bits is flipped independently 
with probability $1/n$. Finally, using Chernoff bounds for the geometric distribution, also the probability 
of the fourth event is bounded by $1-2^{-\Omega(\ell)}$.
\end{proof}

We now present the positive result, showing that the \muoneea both with uniform and inverse binary tournament 
selection optimize $\sufpar$ efficiently for sufficiently large~$\mu$.

\begin{theorem}
Let $\mu=\sqrt{n}$. Then the \muoneea both with uniform and with inverse binary tournament parent selection 
optimize $\sufpar$ in time $O(n^{2.15})$ with probability $1-2^{-\Omega(n^{0.15})}$.
\end{theorem}

\begin{proof}
The proof strategy is again inspired by the corresponding result for the function~$f$ in \cite{WittFamily}. We 
consider a phase of length $s\coloneqq 22\ell n^{1.15}$ and show that w.o.p. 
the number of leading suffix ones rises to its maximum~$\ell$ before 
an individual with at least $2m/3$ prefix ones is created. From there the global optimum is reachable within $O(n^{2.15})$ 
generations w.o.p.

In more detail, we first consider the {\color{black}growth} of the potential~$L$, defined as the maximum number of leading suffix ones 
within the population. During this analysis, we assume that no individual with more than $2m/3$ prefix ones is created within the phase. 
As we will prove below, 
this assumption is valid for $s$ generations with probability $1-2^{-\Omega(n^{0.98{\color{black}\overline{3}}})}$. Given a current potential~$L$, we call an individual 
top if it has $L$ \emph{or more} leading suffix ones. We consider the grow of the number of top individuals, noting that this number is non-decreasing 
since only individuals of lowest fitness are removed and the factor $n^2$ in the definition of $\sufpar$  gives higher fitness to 
individuals with more leading suffix ones regardless of the prefix ones. Let $T$ be the number of top individuals. The probability 
of increasing this number by choosing and cloning such an individual is at least 
$\frac{T}{\mu}(1-1/n)^n \ge \frac{T}{2e\mu}$ for uniform parent selection and 
at least $\frac{T^2}{\mu^2}(1-1/n)^n \ge \frac{T^2}{2e\mu^2}$ for inverse binary tournament selection since it suffices to choose two top individuals 
for the tournament. Hence, by standard argument for takeover times as in \cite{WittFamily}, the expected time 
to reach $\mu$ top individuals is at most
\[
2e\mu\sum_{T=1}^\mu 1/T \le 2e\mu(\ln \mu+1)
\]
for uniform
and 
\[
2e\mu^2\sum_{T=1}^\mu 1/T^2 \le 2e\mu^2 \frac{\pi^2}{6} \le 9\mu^2
\]
for inverse binary tournament selection. We assume that $n$, and thereby~$\mu$, is large enough so that $2e\mu(\ln\mu+1) \le 9\mu^2$ holds. Given $\mu$ top 
individuals and pessimistically assuming the $L$-value has not increased in the meantime, the probability of increasing the $L$-value 
is at least $1/n$ since only the leftmost suffix zero has to flip. Hence, altogether, the expected time to increase the 
$L$-value is at most $9\mu^2+n=10n$ by our choice 
of~$\mu$ and the 
expected time to reach $\ell$ suffix ones in all individuals
 is at most $10\ell n+9\mu^2 \le 11\ell n$ (still assuming no individual with at least $2m/3$ prefix ones). By applying Markov's inequality 
and repeating this argumentation for $n^{0.15}$ subphases of length $20\ell n$, we obtain 
that the probability of reaching $\ell$ suffix ones for all 
individuals within 
$22\ell n^{1.15} = O(n^{1.484})$ generations is at least  $1-2^{-\Omega(n^{0.15})}$. From this point on the algorithm proceeds as if 
it was optimising \onemax on the $m$ prefix bits only. Adjusting the preceding analysis (using $m$ instead of $\ell$ improvements, each pessimistically 
happening with probability at least $1/(2en)$), 
the remaining time to reach the global 
optimum is easily bounded by $O(m n^{1.15})=O(n^{2.15})$ with probability $1-2^{-\Omega(n^{0.15})}$.

It remains to prove that within $s$ generations no individual with at least $2m/3$ prefix ones is created w.o.p.
For this purpose, 
we use a straightforward generalization of the family tree argument from \cite{WittFamily}, pessimistically assuming that each vertex in the tree at time~$t$ twice 
has the chance of~$1/\mu$ to receive a new child. This caters for the inverse 
binary tournament selection and increases the bound on the depth from Lemma~2 in \cite{WittFamily} 
only by a factor of~$2$; hence the probability that the depth $D(t)$ at time~$t$ is $6t/\mu$ or bigger is $2^{-\Omega(t/\mu)}$. We use this statement 
with $t\coloneqq s$, implying that the probability of the depth being at least $d^*\coloneqq 22\ell n^{1.15}/\mu = 22n^{0.98\overline{3}}$ is  $2^{-\Omega(n^{0.98{\color{black}\overline{3}}})}$. Hence, 
in order to reach $2m/3$ prefix ones within $s$ steps, either the tree must have depth greater than $d^*$, or it must contain a so-called bad path of length 
at most~$d^*$ along which the number of prefix ones is increased by at most $m/12$. Here we assume 
that all initial individuals have at most $7m/12$ prefix ones, which 
happens with probability $1-2^{-\Omega(n)}$ by Chernoff bounds and a union bound over $\sqrt{n}$ individuals.

 The probability 
of increasing the number of ones by at least $m/12 = \Omega(n)$ 
on a fixed path of length $d^*=o(n)$ in the family tree is $2^{-\Omega(n)}$ by Chernoff bounds since such a path 
is an independent sequence of mutations without selection. Only if this increase happens, the path is bad. 
By considering all possible paths in a union bound, 
the probability of a bad path emerging in~$s$ generations is at most 
\begin{align*}
\sum_{k=1}^{d^*} \binom{s}{k} \left(\frac{2}{\mu}\right)^k 2^{-\Omega(n)} & = \sum_{k=1}^{d^*} \binom{d^*\mu}{k} \left(\frac{2}{\mu}\right)^k 
2^{-\Omega(n)} \\ 
& \le \sum_{k=1}^{d^*} \left(\frac{ed^*\mu}{k}\right)^k \left(\frac{2}{\mu}\right)^k 2^{-\Omega(n)} = \sum_{k=1}^{d^*} \left(\frac{2ed^*}{k}\right)^k 2^{-\Omega(n)}  \\
& \le d^* \left(\frac{2ed^*}{d^*}\right)^{d^*} 2^{-\Omega(n)}  \le 2^{O(d^*)} 2^{-\Omega(n)} = 2^{-\Omega(n)},
\end{align*}
where the first inequality used the well-known estimate $\binom{n}{k}\le (en/k)^k$, the second one 
that the terms of the sum are monotone increasing within the considered range of~$k$ and the final 
estimation that $d^*=o(n)$. 
Summing up all failure probabilities in a union bound completes the proof.
\end{proof}


\newcommand{\maxsat}{\textsc{MaxSat}\xspace}
\newcommand{\mknap}{\textsc{MKP}\xspace}

We now present some experimental results to provide evidence that the inverse selection operators allow to identify both optima with reasonably high probability,
while uniform selection can only identify the local optimum for small population sizes (including the \ooea) and only the global optimum for large enough population sizes.

To illustrate the behaviour of the considered selection mechanisms, we have
performed 100 independent runs of each mechanism and population size on $\sufpar$
with $n=1000$ (and hence $\ell=10$).

We define four possible outcomes for the runs: converging on the global optimum
(OPT) without having constructed the local optimum (LOC), converging on the
local optimum without having constructed the global optimum, constructing both
optima before converging, or exceeding the computational budget\footnote{For
practical reasons, we consider the global optimum to be unreachable by mutation
of the local optimum, and sample waiting times for inverse tournament selection
to select a viable ancestor when the majority of the population is composed of
individuals on the local optimum. A timeout is reported when these waiting
times cause the total optimisation time to exceed $2^{63}-1$, or when the
inverse elitist mechanism has constructed the local optimum while
individuals with higher fitness on the branch leading to the global optimum
still exist in the population (as these can neither be removed nor chosen for
reproduction).}
without converging (a possibility for the Inverse Elitist and Inverse
Tournament selection mechanisms).

Figure~\ref{fig:sufpar-n1000a} shows the proportion of the runs achieving
specific outcomes. As expected, the (1+1)~EA fails to find the global optimum
in all 100 runs, while increasing the population size allows both uniform and
inverse tournament mechanisms to find the global optima. Additionally, larger
inverse tournament sizes enable the algorithm to construct both optima.
For large enough population and tournament sizes (see eg., $\mu=2n$ and $s=10$) both optima are identified in each run.
For smaller population sizes either of the optima can be identified with reasonable probability (see eg., $\mu=50$ and $s=6$), hence a restart strategy is effective.

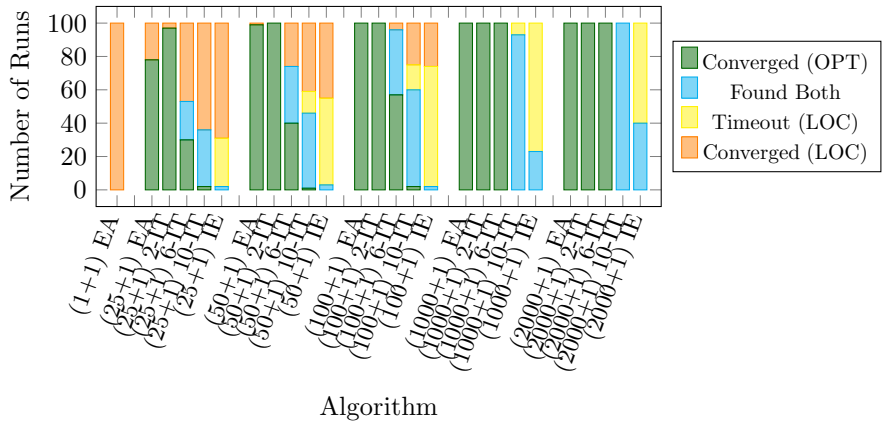
\begin{figure}[t]
\centering
\begin{tikzpicture}
\def\dataTSV{\plotdatapath{sufpar-n1000a.tsv}}
\pgfplotstableread[col sep=tab]{\dataTSV}{\datatable}
\begin{axis}[
width=0.75\linewidth,
height=0.35\linewidth,
legend style={nodes={scale=0.8, transform shape}, at={(1.02,0.5)}, anchor=west},
legend columns=1,
xlabel={Algorithm},
ylabel={Number of Runs},
ybar stacked=plus,
try min ticks=5,
mark size=1.5pt,mark options={solid},mark repeat=100,
xtick=data,
xticklabels from table={\datatable}{[index]0},
x tick label style={rotate=70,anchor=east, font=\footnotesize},
bar width=5pt,
every axis plot/.append style={fill,fill opacity=0.5},
enlarge x limits=0.04,
every x tick label/.append style={font=\small},
]
\addplot[green!40!black] table[x expr=\coordindex,y index=1] {\datatable};
\addplot[cyan] table[x expr=\coordindex,y index=3] {\datatable};
\addplot[yellow] table[x expr=\coordindex,y index=4] {\datatable};
\addplot[orange] table[x expr=\coordindex,y index=5] {\datatable};

\legend{Converged (OPT), Found Both, Timeout (LOC), Converged (LOC)}
\end{axis}
\end{tikzpicture}
	\caption{Outcomes of 100 runs of algorithms using various selection operators and population sizes on $\sufpar$ with $n=1000$ (and $\ell = 10$).}
	\label{fig:sufpar-n1000a}
\end{figure}

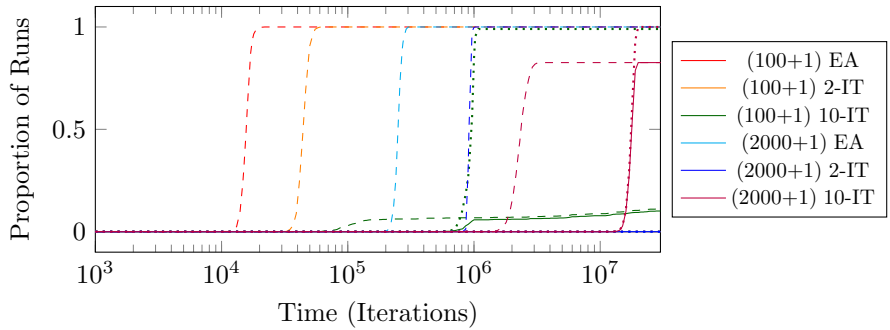
\begin{figure}[t]
\centering
\begin{tikzpicture}
\def\dataTSV{\plotdatapath{sufpar-n1000t-both-u.tsv}}
\pgfplotstableread[col sep=tab]{\dataTSV}{\databoth}
\def\dataTSV{\plotdatapath{sufpar-n1000t-best-u.tsv}}
\pgfplotstableread[col sep=tab]{\dataTSV}{\databest}
\def\dataTSV{\plotdatapath{sufpar-n1000t-local-u.tsv}}
\pgfplotstableread[col sep=tab]{\dataTSV}{\dataloc}


\begin{axis}[
width=0.75\linewidth,
height=0.40\linewidth,
legend style={nodes={scale=0.8, transform shape}, at={(1.02,0.5)}, anchor=west},
legend columns=1,
xmode=log,
xlabel={Time (Iterations)},
ylabel={Proportion of Runs},
enlarge x limits=0,
gopt/.style={dashed, forget plot},
lopt/.style={dotted,thick,forget plot},
]

\addplot[red]      table[x index=0,y index=1] {\databoth};
\addplot[red,gopt] table[x index=0,y index=1] {\databest};
\addplot[red,lopt] table[x index=0,y index=1] {\dataloc};

\addplot[orange]      table[x index=0,y index=2] {\databoth};
\addplot[orange,gopt] table[x index=0,y index=2] {\databest};
\addplot[orange,lopt] table[x index=0,y index=2] {\dataloc};

\addplot[green!40!black]      table[x index=0,y index=4] {\databoth};
\addplot[green!40!black,gopt] table[x index=0,y index=4] {\databest};
\addplot[green!40!black,lopt] table[x index=0,y index=4] {\dataloc};

\addplot[cyan]      table[x index=0,y index=9] {\databoth};
\addplot[cyan,gopt] table[x index=0,y index=9] {\databest};
\addplot[cyan,lopt] table[x index=0,y index=9] {\dataloc};

\addplot[blue]      table[x index=0,y index=10] {\databoth};
\addplot[blue,gopt] table[x index=0,y index=10] {\databest};
\addplot[blue,lopt] table[x index=0,y index=10] {\dataloc};

\addplot[purple]      table[x index=0,y index=12] {\databoth};
\addplot[purple,gopt] table[x index=0,y index=12] {\databest};
\addplot[purple,lopt] table[x index=0,y index=12] {\dataloc};

\legend{(100+1) EA, (100+1) 2-IT, (100+1) 10-IT, (2000+1) EA, (2000+1) 2-IT, (2000+1) 10-IT}
\end{axis}
\end{tikzpicture}
\caption{Proportion of 1000 runs which find both the local and global optima
(solid), at least the global optimum (dashed), or at least the local optimum
(dotted) within a certain number of iterations on $\sufpar$ with $n=1000$ (and
$\ell = 10$) for various selection operators and population sizes.}
\label{fig:sufpar-n1000b}
\end{figure}

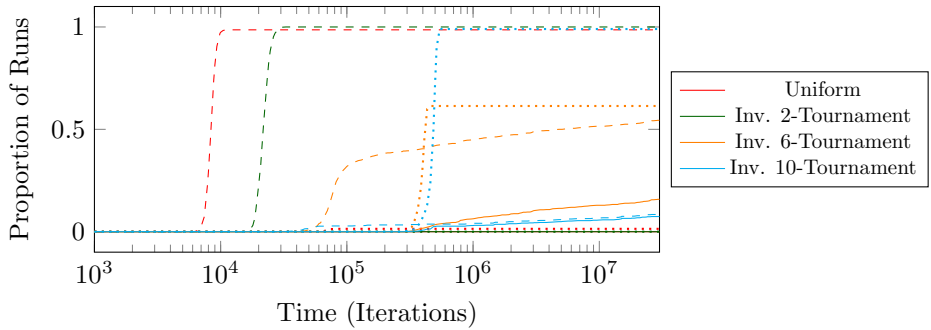
\begin{figure}[t]
\centering
\begin{tikzpicture}
\def\dataTSV{\plotdatapath{sufpar-n1000t-both-u.tsv}}
\pgfplotstableread[col sep=tab]{\dataTSV}{\databoth}
\def\dataTSV{\plotdatapath{sufpar-n1000t-best-u.tsv}}
\pgfplotstableread[col sep=tab]{\dataTSV}{\databest}
\def\dataTSV{\plotdatapath{sufpar-n1000t-local-u.tsv}}
\pgfplotstableread[col sep=tab]{\dataTSV}{\dataloc}


\begin{axis}[
width=0.75\linewidth,
height=0.40\linewidth,
legend style={nodes={scale=0.8, transform shape}, at={(1.02,0.5)}, anchor=west},
legend columns=1,
xmode=log,
xlabel={Time (Iterations)},
ylabel={Proportion of Runs},
enlarge x limits=0,
gopt/.style={dashed, forget plot},
lopt/.style={dotted,thick,forget plot},
]
\addplot[red]      table[x index=0,y index=5] {\databoth};
\addplot[red,gopt] table[x index=0,y index=5] {\databest};
\addplot[red,lopt] table[x index=0,y index=5] {\dataloc};

\addplot[green!40!black]      table[x index=0,y index=6] {\databoth};
\addplot[green!40!black,gopt] table[x index=0,y index=6] {\databest};
\addplot[green!40!black,lopt] table[x index=0,y index=6] {\dataloc};

\addplot[orange]      table[x index=0,y index=7] {\databoth};
\addplot[orange,gopt] table[x index=0,y index=7] {\databest};
\addplot[orange,lopt] table[x index=0,y index=7] {\dataloc};

\addplot[cyan]      table[x index=0,y index=8] {\databoth};
\addplot[cyan,gopt] table[x index=0,y index=8] {\databest};
\addplot[cyan,lopt] table[x index=0,y index=8] {\dataloc};

\legend{Uniform, Inv. 2-Tournament, Inv. 6-Tournament, Inv. 10-Tournament}
\end{axis}
\end{tikzpicture}
\caption{Proportion of 1000 runs which find both the local and global optima
(solid), at least the global optimum (dashed), or at least the local optimum
(dotted) within a certain number of iterations on $\sufpar$ with $n=1000$ (and
$\ell = 10$) for various selection operators with $\mu=50$.}
\label{fig:sufpar-n1000b-mu50}
\end{figure}

Figures \ref{fig:sufpar-n1000b} and \ref{fig:sufpar-n1000b-mu50} give an
estimate of the probability that a particular type of optimum (or both optima)
have been constructed after a certain number of fitness iterations. With
sufficiently large population and tournament sizes, the inverse tournament
selection mechanism is shown to be able to find both the local and global
optima with reasonable probability. This indicates that it could beneficially
be combined with a restart strategy to reliably find both optima, while our
theoretical results show that the (1+1) EA would require an exponential number
of restarts to reliably find the global optimum. Additionally, our experimental
results show that the $(\mu+1)$ EA does not 
find the local optimum
when $\mu$ is large. 
Hence, if the fitness values of the
optima were exchanged, the $(\mu+1)$~EA would not have found the global optimum in any of the runs.

\section{
NP-Complete Problems} \label{sec:npcomplete}
In the previous section we have shown good explorative performance of the \mupoea with lower selective pressure than uniform on various multimodal  pseudo-Boolean benchmark functions with different characteristics.
In this section we verify the global exploration effectiveness of lowering the selective pressure 
for two
classical NP-Complete
combinatorial optimisation problems, \maxsat and the multidimensional knapsack
problem (\mknap).

\subsection{MaxSat}
In the \maxsat problem, a set of clauses over $n$ Boolean variables is
given, and the objective is to find a truth-value assignment which maximises
the number of satisfied clauses. 

To evaluate the performance of Inverse Elitist and Inverse Tournament selection
mechanisms, we use the uniform Random-3-SAT instances from SATLIB, which ``tend
to be particularly hard for both systematic SAT solvers and stochastic local
search algorithms'' \cite{Hoos2000satlib}.
Solutions are represented as an $n$-bit string, and interpreted as truth-value
assignments by setting $x_i$ to true if and only if bit $i$ is set to 1. The
fitness value of a solution is simply the number of clauses it satisfies.

In particular, we focus on the instances in the \texttt{uf100} and
\texttt{uf250} classes (composed of 1000 and 100 instances, respectively),
which are satisfiable uniform Random-3-SAT instances with $n=100$ and $n=250$
variables respectively. We perform 100 independent runs of each algorithm on
each instance in the problem class, and record the number of iterations it
takes to produce a solution which satisfies all clauses.

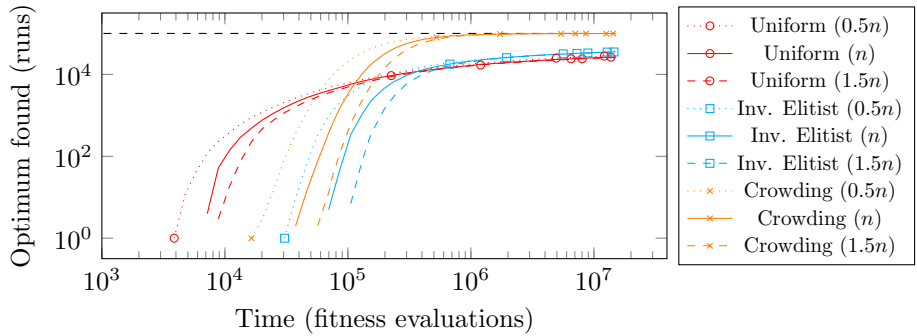
\begin{figure}[t]
\centering
\begin{tikzpicture}
\def\dataTSV{\plotdatapath{maxsat-uf100-solvedruns.tsv}}
\begin{semilogxaxis}[ymode=log,
width=0.75\linewidth,
height=0.40\linewidth,
legend style={nodes={scale=0.8, transform shape}, at={(1.02,0.5)}, anchor=west},
legend columns=1,
xlabel={Time (fitness evaluations)},
ylabel={Optimum found (runs)},
try min ticks=5,
xmin=1000,
mark size=1.5pt,mark options={solid},mark repeat=100,
]
\draw [dashed] (0,10^5) -- (10^7,10^5);

\addplot[mark phase=00,red,mark=o,dotted] table[x index=0,y index=1] {\dataTSV};
\addplot[mark phase=20,red,mark=o] table[x index=0,y index=9] {\dataTSV};
\addplot[mark phase=40,red,mark=o,dashed] table[x index=0,y index=17] {\dataTSV};

\addplot[mark phase=00,cyan,mark=square,dotted] table[x index=0,y index=8] {\dataTSV};
\addplot[mark phase=20,cyan,mark=square] table[x index=0,y index=16] {\dataTSV};
\addplot[mark phase=40,cyan,mark=square,dashed] table[x index=0,y index=24] {\dataTSV};

\addplot[mark phase=00,orange,mark=x,dotted] table[x index=0,y index=2] {\dataTSV};
\addplot[mark phase=20,orange,mark=x] table[x index=0,y index=10] {\dataTSV};
\addplot[mark phase=40,orange,mark=x,dashed] table[x index=0,y index=18] {\dataTSV};

\legend{Uniform ($0.5n$),Uniform ($n$),Uniform ($1.5n$),
Inv. Elitist ($0.5n$),Inv. Elitist ($n$),Inv. Elitist ($1.5n$),
Crowding ($0.5n$),Crowding ($n$),Crowding ($1.5n$)}
\end{semilogxaxis}
\end{tikzpicture}
	\caption{Number of runs finding a satisfying solution for \maxsat (combined over 100 independent runs on each of the 1000 instances in the SATLIB \texttt{uf100} problem class) as a function of the elapsed number of fitness evaluations using various parent selection mechanisms and population sizes.}
	\label{fig:maxsat-uf100b}
\end{figure}

Figure~\ref{fig:maxsat-uf100b} compares the number of successful runs (i.e.
those that find a satisfying truth-value assignment) of the Uniform and Inverse
Elitist selection mechanisms, as well as Deterministic Crowding, 
for varying population sizes
as a function of the available fitness evaluation budget on the \texttt{uf100}
problem class.

Uniform selection is the quickest mechanism to produce satisfying solutions for
some instances. However, as the number of fitness evaluations increases, it
achieves fewer successful runs than both Inverse Elitist selection and
Deterministic Crowding. Additionally, while smaller populations (i.e.
$\mu=n/2$) generally produce satisfying solutions faster, all three considered
population sizes seem to converge to similar values for each algorithm as the
fitness evaluation budget increases. This suggests that there are many initial
search points from which the a global optimum can be reached within the
allocated number of fitness evaluations, hence an explanation of why the parallel (1+1)~EAs due to deterministic crowding perform best.

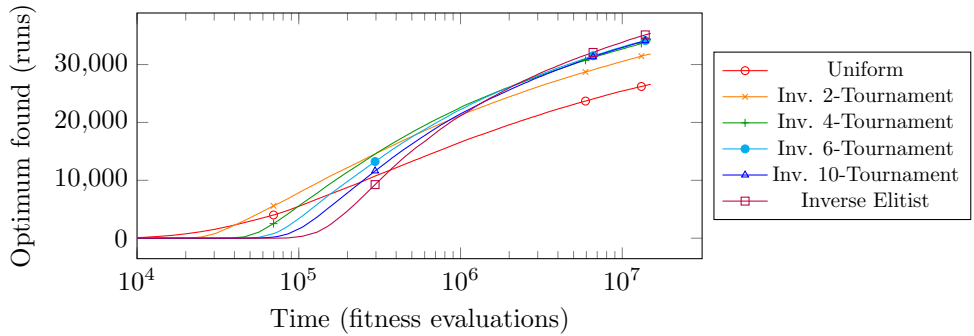
\begin{figure}[t]
\centering
\begin{tikzpicture}
\def\dataTSV{\plotdatapath{maxsat-uf100-solvedruns.tsv}}
\begin{semilogxaxis}[
width=0.75\linewidth,
height=0.40\linewidth,
legend style={nodes={scale=0.8, transform shape}, at={(1.02,0.5)}, anchor=west},
legend columns=1,
xlabel={Time (fitness evaluations)},
ylabel={Optimum found (runs)},
scaled y ticks = false, 
try min ticks=5,
xmin=10000, 
mark size=1.5pt,mark options={solid},mark repeat=100,
]

\addplot[mark phase=10,red,mark=o] table[x index=0,y index=9] {\dataTSV};
\addplot[mark phase=10,orange,mark=x] table[x index=0,y index=11] {\dataTSV};
\addplot[mark phase=10,green!60!black,mark=+] table[x index=0,y index=13] {\dataTSV};
\addplot[mark phase=20,cyan,mark=*] table[x index=0,y index=14] {\dataTSV};
\addplot[mark phase=20,blue,mark=triangle] table[x index=0,y index=15] {\dataTSV};
\addplot[mark phase=20,purple,mark=square] table[x index=0,y index=16] {\dataTSV};

\legend{Uniform, Inv. 2-Tournament, Inv. 4-Tournament, Inv. 6-Tournament, Inv. 10-Tournament, Inverse Elitist}
\end{semilogxaxis}
\end{tikzpicture}
	\caption{Number of runs finding a satisfying solution for \maxsat (combined over 100 independent runs on each of the 1000 instances in the SATLIB \texttt{uf100} problem class) as a function of the elapsed number of fitness evaluations using various parent selection mechanisms and $\mu = n = 100$.}
	\label{fig:maxsat-uf100a}
\end{figure}

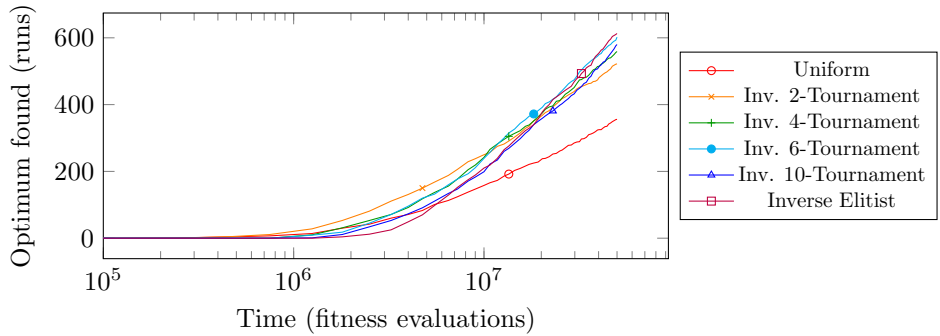
\begin{figure}[t]
\centering
\begin{tikzpicture}
\def\dataTSV{\plotdatapath{maxsat-uf250-solvedruns.tsv}}
\begin{axis}[xmode=log,
width=0.75\linewidth,
height=0.40\linewidth,
legend style={nodes={scale=0.8, transform shape}, at={(1.02,0.5)}, anchor=west},
legend columns=1,
xlabel={Time (fitness evaluations)},
ylabel={Optimum found (runs)},
try min ticks=5,
xmin=100000, 
mark size=1.5pt,mark options={solid},mark repeat=100,
each nth point=4,
]

\addplot[mark phase=20,red,mark=o] table[x index=0,y index=9] {\dataTSV};
\addplot[mark phase=10,orange,mark=x] table[x index=0,y index=11] {\dataTSV};
\addplot[mark phase=20,green!60!black,mark=+] table[x index=0,y index=13] {\dataTSV};
\addplot[mark phase=25,cyan,mark=*] table[x index=0,y index=14] {\dataTSV};
\addplot[mark phase=30,blue,mark=triangle] table[x index=0,y index=15] {\dataTSV};
\addplot[mark phase=40,purple,mark=square] table[x index=0,y index=16] {\dataTSV};

\legend{Uniform, Inv. 2-Tournament, Inv. 4-Tournament, Inv. 6-Tournament, Inv. 10-Tournament, Inverse Elitist}
\end{axis}
\end{tikzpicture}
	\caption{Number of runs finding a satisfying solution for \maxsat (combined over 100 independent runs on each of the 100 instances in the SATLIB \texttt{uf250} problem class) as a function of the elapsed number of fitness evaluations using various parent selection mechanisms and $\mu = n = 250$.}
	\label{fig:maxsat-uf250a}
\end{figure}

Figures~\ref{fig:maxsat-uf100a} and \ref{fig:maxsat-uf250a} respectively compare the
performance of the uniform, inverse tournament, and inverse elitist selection
mechanisms with $\mu=n$ on the \texttt{uf100} and \texttt{uf250} problem
classes (i.e. instances with $n=100$ and $n=250$ respectively).
With a sufficient fitness evaluation budget, the inverse
tournament and inverse elitist selection mechanisms outperform uniform
selection, being able to find the satisfying assignments more frequently. As
expected, large inverse tournaments behave similarly to the deterministic
Inverse Elitist selection mechanism (being slower to find the first satisfying
assignments, but eventually solving more instances), while smaller inverse
tournaments behave similarly to the Uniform selection mechanism (quicker to
find solutions on easier instances, but less successful in the longer runs).
Notably, there exists a range of fitness evaluation budgets where the inverse
tournament mechanisms outperform both the Uniform and Inverse Elitist
mechanisms.
This suggests that on these problem instances, the inverse selection mechanisms
are able to preserve a greater amount of solution diversity during the
optimisation process at the cost of a reduced hillclimbing speed. This
increased diversity enables the global optimum to be found more reliably
compared to the uniform selection mechanism, which may converge on an
initially-promising branch leading to a local optimum that is difficult to
escape from sooner in the optimisation process.

To confirm that the inverse selection mechanisms perform better than Uniform
selection on the \maxsat instances, and that the difference observed in the
experiments is statistically significant, we apply the Wilcoxon signed-rank
test \cite{Wilcoxon45}, and use the Holm-Bonferroni method \cite{Holm79} to
account for multiple comparisons with a significance level $\alpha = 0.01$. We
consider two algorithm performance metrics: the fitness of the best solution
constructed within the allocated budget ($B=15,\!000,\!000$ and
$B=50,\!000,\!000$ fitness evaluations for the \texttt{uf100} and
\texttt{uf250} instance classes respectively), and the time taken to find a
satisfying variable assignment (using the budget as an upper bound if no such
solution is found during a run). When performing the comparisons between the
performance of two algorithms on a problem class (i.e. \texttt{uf100} and
\texttt{uf250}), runs on all instances within the problem class are used, and
runs on the same instance are paired; thus each comparison on the
\texttt{uf100} problem class uses the results of $100 \times 1000 = 10^5$ runs
of each algorithm, and each comparison on the \texttt{uf250} problem class uses
the results of $100 \times 100 = 10^4$ runs of each algorithm. For both problem
classes, and all three considered population sizes ($\mu \in \{n/2, n,
3n/2\}$), the performed two-sided statistical tests confirm that there was a
significant difference in both best solution fitness and average runtime
between the Uniform selection mechanism and other selection mechanisms
considered, while one-sided statistical tests confirm that the inverse
tournament selection mechanisms outperform the Uniform selection mechanism (i.e. find better solutions after at
most $B$ iterations, and find the optimum solutions faster within a budget of
$B$ iterations).


\subsection{Multidimensional Knapskack}

In the Multidimensional Knapsack  problem (\mknap), the objective is to select a
subset of the $n$ available items which maximises the value of the selected
{\color{black}items}, subject to constraints in $m$ dimensions. More formally, given
capacities $b_i$ ($1\leq i \leq m$), item rewards $p_j$ ($1\leq j\leq n$), and
weights $r_{ij}$, the goal is to choose $x_j\in\{0,1\}$ to maximise
$\sum_{i=1}^n p_j x_j$ subject to $\sum_{j=1}^{n} r_{ij} x_j \leq b_i$ for all
$i=1,\ldots, m$.

We represent solutions as an $n$-bit string, with bit $j$ determining whether
item $j$ is included in the selection or not. To guide the EAs toward valid
solutions (i.e. those that satisfy the capacity constraints), we define the
fitness of the solutions as follows: let $W = 1+\sum_{j=1}^n p_j$, then
$$ f(x) = \left(\sum_{j=1}^n x_j p_j\right) + W\cdot \sum_{i=1}^{m} \min\left(0, b_i - \sum_{j=1}^{n} r_{ij}{x_j}\right) . $$
The fitness of solutions which satisfy the constraints is the sum of the
rewards $p_j$ of the selected items. If any constraint is violated, the
fitness will be negative (as the rightmost sum will produce a negative value),
and will improve if an item which contributes to a violating constraint is
removed.

To evaluate the performance of the selection mechanisms on \mknap, we use the
\texttt{mknap} instances from the OR-Library \cite{beasley1990or}, originally
constructed by Chu and Beasley \cite{Chu1998} to evaluate the performance of a
genetic algorithm on the problem. This set contains 270 instances in total,
and consists of 10 instances for each combination of $m \in \{5, 10, 30\}$, $n\in\{100, 250, 500\}$
and constraint tightness ratios $\alpha \in \{0.25, 0.5, 0.75\}$. As the
instances were randomly generated, the fitness value of the global optimum is
not known. Hence, instead of considering the time required to reach the global
optimum, we record the fitness value of a solution found after a certain number
of fitness evaluations, and normalise these values by dividing by the fitness
value of the best known solution for the instance (combining the best fitness
values listed in OR-Library with those found during the runs).

Figures~\ref{fig:mknap-group-m5}, \ref{fig:mknap-group-m10}, \ref{fig:mknap-group-m30}
show the results aggregated over 100 independent
runs for each algorithm and problem instance, giving all algorithms a budget of
$30,\!000,\!000$ fitness evaluations. We note that the larger Inverse Tournament sizes,
such as the Inverse-10-Tournament tend to produce solutions of higher fitness
than the standard $(1+1)$ EA, the standard $(\mu+1)$ EA, and $\mu$ parallel
runs of the standard $(1+1)$ EA (each given $30,\!000,\!000/\mu$ fitness evaluations).
The fully Inverse Elitist selection mechanism does not perform as well.
This may be explained by the mechanism requiring a larger mutation to occur in
order to escape from any local optimum it constructs before any higher-fitness
individuals can be considered.

\begin{figure*}
	\input{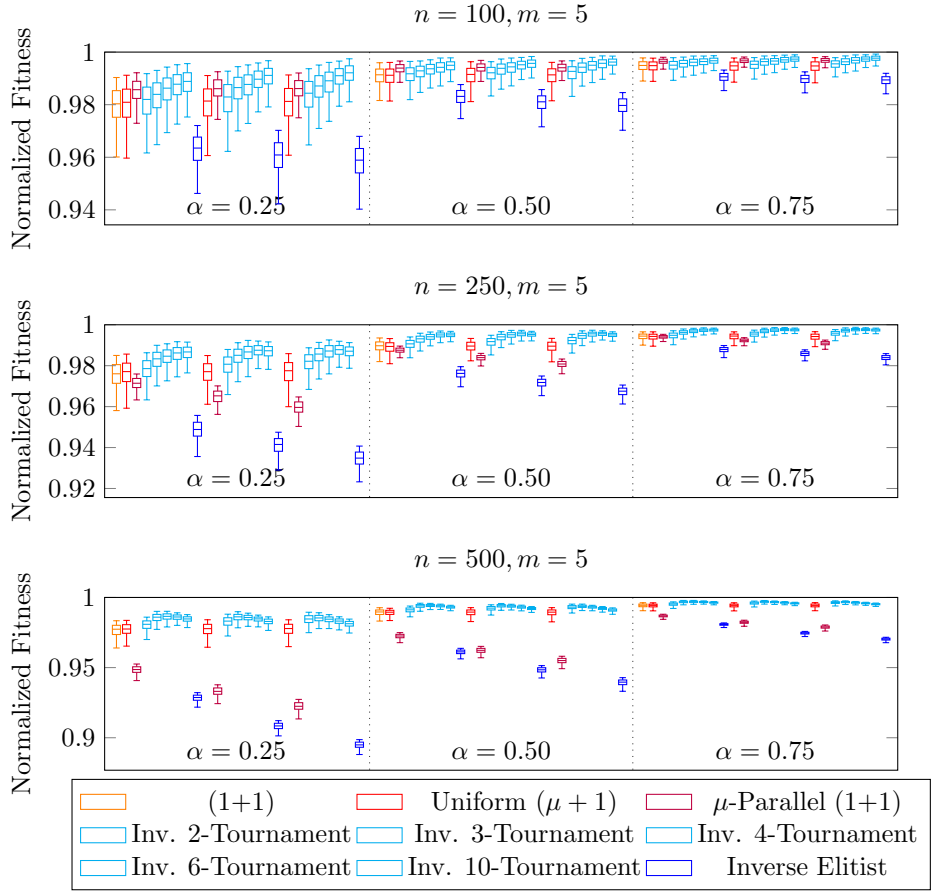}
\caption{Solution quality produced by a $(\mu+1)$~EA after $30,\!000,\!000$
iterations, using various selection mechanisms and choices of $\mu$ on
5-dimensional \mknap instances from OR-LIB with varying constraint
tightness $\alpha$. Each block repeats the population-employing algorithms
using $\mu=n/2$, $\mu=n$, and $\mu=3n/2$ in that order.}
\label{fig:mknap-group-m5}
\end{figure*}
\begin{figure*}
	\input{figures/mknap-grouped-m10.tex}
\caption{Solution quality produced by a $(\mu+1)$~EA after $30,\!000,\!000$
iterations, using various selection mechanisms and choices of $\mu$ on
10-dimensional \mknap instances from OR-LIB with varying constraint
tightness $\alpha$. Each block repeats the population-employing algorithms
using $\mu=n/2$, $\mu=n$, and $\mu=3n/2$ in that order.}
\label{fig:mknap-group-m10}
\end{figure*}
\begin{figure*}
	\input{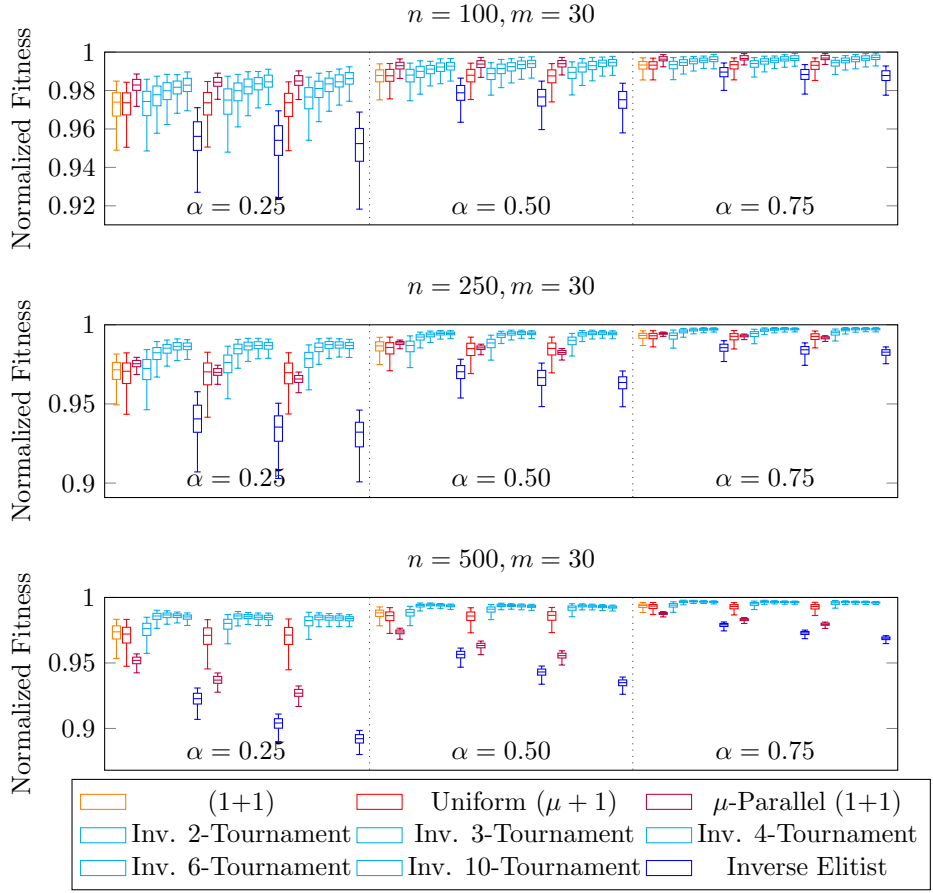}
\caption{Solution quality produced by a $(\mu+1)$~EA after $30,\!000,\!000$
iterations, using various selection mechanisms and choices of $\mu$ on
30-dimensional \mknap instances from OR-LIB with varying constraint
tightness $\alpha$. Each block repeats the population-employing algorithms
using $\mu=n/2$, $\mu=n$, and $\mu=3n/2$ in that order.}
\label{fig:mknap-group-m30}
\end{figure*}

To evaluate whether the observed differences in solution fitness after
$30,\!000,\!000$ iterations are significant, and to confirm that the Inverse
Tournament selection mechanisms outperform Uniform selection and $\mu$ Parallel
runs mechanisms, we apply the same Wilcoxon signed-rank test \cite{Wilcoxon45}, and Holm-Bonferroni method \cite{Holm79} 
as in the previous section for \maxsat again using a significance level of $\alpha = 0.01$. 
When performing the
comparisons between the performance of two algorithms, runs on all available
instances are used, and runs on the same instance are paired; thus, each
comparison is based on the results of $100 \times 270 = 27,\!000$ runs of each
algorithm. The performed two-sided tests confirm that for all considered
population sizes ($\mu \in \{n/2, n, 3n/2\}$) there is a significant difference
between the solution fitness produced when using Uniform selection and all
other algorithms, as well as between the $\mu$ parallel runs of a $(1+1)$ EA
and all other algorithms, while the performed one-sided tests confirm that the
Inverse Tournament mechanisms produce solutions with higher fitness (compared
to Uniform selection and $\mu$ parallel (1+1) runs) at the end of the
$B$-iteration budget.

For these problem instances and fitness evaluation budget, the poor performance
of the $\mu$ parallel runs of the $(1+1)$ EA, particularly as $\mu$ increases,
may be explained by the existence of many basins of attraction toward local
optima that are difficult to escape from. As the $\mu$ $(1+1)$ EAs do not
interact during the optimisation process, once a $(1+1)$ EA has reached a local
optimum, its remaining computational budget is spent waiting for a mutation
that escapes the local optimum. If a large mutation is required to escape,
waiting for it to occur may be wasteful when compared to using fitness
evaluations to explore the regions of the search space around higher-fitness
solutions found by the other search trajectories.

The Inverse Tournament selection mechanisms are able to remove local optima
from the population by selecting a better-fitness individual for reproduction,
re-allocating the remaining fitness evaluations to explore a higher-fitness
region of the search space, enabling the mechanism to more gracefully cope with
increasing population sizes. The size of the inverse tournament balances the
ease of such re-allocations with the diversity-preserving effects of inverse
selection: small tournaments are more likely to converge the population on a
single trajectory sooner, while larger tournaments may use more fitness
evaluations to explore the region around a local optimum before ejecting it
from the population.

\section{Discussion}

\textcolor{black}{In the previous sections we have shown several advantages of using lower selective pressure than uniform for standard steady-state EAs i.e., that only use mutation.
In this section we present a discussion of how previous literature has argued that uniform parent selection induces too high selective pressure also in steady-state GAs that use crossover as well as mutation.} 

Striving for a good balance between exploration and exploitation has arguably been the hottest topic in evolutionary computation since its conception.
Already in his seminal book, John Holland introduced and analysed two forms of reproductive schemes to serve as a basis for the field of genetic algorithms~\cite{Holland1992}.
One was the canonical generational model, where $\mu$ individuals were selected according to fitness proportional selection and the offspring replaced the entire population, independent of their quality.
The other scheme used {\it overlapping generations} where only one new offspring was created in each generation: an individual was selected for reproduction with fitness proportional selection and an individual was removed uniformly at random
from the parent population to make space for the new offspring. Hence, individuals would have an expected lifetime of $\mu$ generations rather than only one generation as in generational models. The idea was that if the average fitness of the population didn't change much over the lifetime of an individual $i$, then the expected number of offspring of that individual would be the same as in the generational model (i.e., $f(i) / \bar{f}$, which was the basis for Holland's schema theorem).

In practice, the average fitness of the population changes  over the lifetime of individuals in the overlapping model. An empirical study by Kenneth De Jong showed how the use of overlapping generations introduced genetic drift (i.e., faster takeover times) which increased in severity with the decrease in size of the offspring population (i.e., $\lambda=1$ leads to the most severe genetic drift)~\cite{DeJongSarma1993}. As a result, the generational model was preferred by researchers and became the standard as well as the backbone of the canonical genetic algorithm~\cite{Goldberg1989}. Since fitness proportional selection was applied to select parents, fitness-scaling functions had to be used to distinguish between individuals and avoid premature convergence once all the population had similar fitness values.

In 1985 James Baker reported experiments where the selection probability of individuals was allocated according to their rank in the population rather than according to their fitness relative to the population's average fitness.
Rank-based selection not only resolved the scaling problems due to fitness proportional selection (which had previously also been avoided by Brindle's Tournament selection~\cite{DebGoldberg1991}) but, more importantly, allowed for an easy way to control the selective pressure and keep it constant throughout the optimisation process~\cite{Whitley1989}. Rank selection allows to explicitly set the expected number of offspring of each individual as a function of their rank in the population. 
Baker's intention was to slow down convergence to achieve more accurate optimisation. With the introduction of the Genitor (GENetic ImplemenTOR) algorithm, the first of what Syswerda termed ``steady-state'' GAs~\cite{Syswerda1991}, 
the opportunity to use steady-state schemes to gain control over the conflict between exploration and exploitation was also envisaged. Whitley states that ``It can be argued that there are only two primary factors in genetic search: population diversity and selective pressure. These two factors are inversely related [...] As selective pressure is increased, the search focuses on the top individuals in the population, but because of this ``exploitation'' genetic diversity is lost. Reducing the selective pressure (or using larger population sizes) increases ``exploration'' because more genotypes and thus more schemata are involved in the search [...] [Hence] selective pressure and population diversity should be controlled as directly as possible [...] Selective pressure can be {\it simply} controlled by allocating reproductive trials according to rank''~\cite{Whitley1989}. 

However, the envisaged balance between exploration and exploitation was never achieved with the Genitor algorithm or other steady-state GAs because
the same parent selection operators as used in generational schemes were adopted (typically rank selection with a bias of at least 1.5 in Genitor and fitness-proportional selection by Syswerda)
while at the same time low quality individuals were deleted for replacement (the worst in Genitor, exponential inverse ranking in Syswerda's version). 
Naturally, these choices led to considerably greater genetic drift than that witnessed by De Jong (when, in his overlapping GA, he deleted random individuals) 
as pointed out by an excellent theoretical analysis of the takeover times of various selection operators by Goldberg and Deb~\cite{DebGoldberg1991}. 
In this paper we have shown how the steady-state scheme can indeed be used to largely reduce the exploration/exploitation conflict. 
 While decreasing the selective pressure to uniform or even lower is indeed impractical (if not impossible) in generational schemes~\cite{LehreEtAl2018}, it can be easily implemented in steady-state algorithms.

In particular, the artificially introduced {\it elitism} in steady-state GAs {\it protects} the best individuals so long as better ones are not identified. Hence, there is no hurry to exploit the best individuals before they disappear: as long as they really are good, they will not be removed from the population.  
This simple insight sheds light on how the conflict between exploration and exploitation may be largely reduced with steady-state schemes: there is no need for a perfect balance between exploration and exploitation. Exploration may proceed for as long as desired because the best individuals will always be waiting to be exploited as long as better solutions are not identified first. 
In particular, steady-state algorithms allow for the selective pressure to be {\it reduced} to low levels that are impossible to achieve with the generational model. 
This is indeed the distinguishing power of steady-state EAs and what makes overlapping generations special. 
In this paper we have provided theoretical and empirical evidence of the advantages that can be gained by exploiting this insight.
\textcolor{black}{ 
To the best of our knowledge no previous studies of reducing the parent selective pressure below uniform have been undertaken.
We have recently discovered that the insight presented in this paper to avoid premature convergence (i.e., preventing the best individuals from reproducing) was briefly mentioned by B{\"a}ck and Hoffmeister in 1991 and termed "left extinctive selection"~\cite{BackHoffmeisterICGA1991}. However, they deemed that the idea might not have any practical relevance and, to the best of our knowledge, was abandoned (the idea is indeed impractical in the context of the ($\mu,\lambda)$ generational GAs	they were using in this particular work: a steady-state scheme is required). On the other hand,  studies on replacement strategies to reduce genetic drift have been undertaken which led to smaller takeover times than in standard steady-state GAs~\cite{Smith2007,BrankeCutaiaDold1999}.}

\section{Conclusion}
In this paper we have analysed the effects of decreasing the parent selection pressure of a standard steady-state EA well below the commonly used uniform selection and have shown the benefits of this decrease without requiring any further algorithmic sophistication.
In particular, we first 
proved that the \mupoea and \muporls evolving populations of arbitrary polynomial size 
via uniform parent selection, or via operators with higher selective pressure, fail
to identify both optima of the \twomax benchmark function with high probability.
On the other hand, we proved that the algorithms are efficient at locating both optima for reasonable population sizes
if Inverse Elitist selection is used instead. An experimental analysis explains why Inverse Tournament selection with small tournament sizes
is even more powerful for global optimisation: for the tested problem sizes it has higher success probabilities for \twomax and is capable of 
escaping local optima of \truncatedtwomax. In particular, in order to escape from local optima some probability of not selecting the worst individuals is crucial.
%

Afterwards we analysed the impressive performance of the low selective pressure algorithms for more sophisticated multi-modal problems generalised from  previously used benchmark functions from the literature which we named \textsc{RidgeWithBranches} and \textsc{TwoGradients}.
The two functions have different characteristics. The former consists of a series of local optima that should be identified one at a time, while the latter consists of two gradients with different steepnesses which should be optimised in parallel.
While the Inverse Tournament \mupoea is efficient for all the problems, we are not aware of any other algorithm, with or without diversity, that can efficiently optimise all the three function classes \twomax, \textsc{RidgeWithBranches} and \textsc{TwoGradients}. In fact we are not aware of any other algorithm that is effective for \textsc{RidgeWithBranches} alone. The latter is an exemplary benchmark function to highlight the explorative capabilities of steady-state EAs with low selective pressure: the algorithm explores one branch at a time, identifying one local optimum after another, and only abandoning each one if it does not turn out to be the global optimum. 

The last section of the paper presented an experimental analysis of the algorithms for the classical \textsc{MaxSat} and Multidimensional Knapsack problems.
For both problems the inverse selection operators outperform uniform selection. However, for the MaxSat instances single trajectory algorithms with restarts outperform all the population based EAs, while for the knapsack instances the opposite is true. 

We conclude with a final comment regarding the optimisation time of the inverse selection algorithms. 
We have presented very effective algorithms for multi-modal optimisation with minimal sophistication (i.e., no additional parameters compared to basic EAs).
For all the problems considered in this paper, we see no reason to use a selection operator with higher pressure than inverse tournaments of constant size.
For difficult multi-modal problems with many optima, extremely fast global optimisation cannot reasonably be expected.
For easier unimodal problems, the selective pressure may be increased by simply decreasing the population size (i.e., if the problem is easy, then large populations are not necessary).
For instance, consider the simple OneMax problem. The expected runtime of the inverse tournament \mupoea is $O(\mu n \ln n)$ which is the same asymptotic expected runtime of the algorithm using uniform selection and deterministic crowding as diversity mechanism~\cite{FriedrichOSWECJ09}.
The expected runtime can be reduced to the best achievable via standard bit mutation-only (i.e.,   $O(n \ln n)$) simply by reducing the population size to a constant (i.e., effectively increasing the selective pressure).
Furthermore, we believe that the standard ($\mu$+1)~GA (i.e., with crossover)  coupled with inverse tournament selection with constant tournament and population sizes is faster than any standard bit mutation-only evolutionary algorithm for OneMax. 
This result may be proven by simply re-calculating the takeover times in the analysis of~\cite{CorusOliveto2019}. In particular, also the result that the algorithm is faster than any unary unbiased 
search heuristic, if identical copies of parents are not re-evaluated, holds.
\textcolor{black}{Indeed, future work should explore theoretically and experimentally the effects of lower selective pressure on more sophisticated algorithms using crossover, more advanced mutation operators~\cite{CorusOlivetoYazdaniFASTAIS20018,CorusOlivetoYazdaniAIJ2019,DoerrGecco2017}, and problem-tailored ones~\cite{GSAT,WalkSAT}.}

{\bf Acknowledgments} 

The authors would like to thank the members of the EC-Theory group in the UK for initial discussions regarding the analysis of 
steady-state EAs with large population sizes for \textsc{TwoMax}.
A special thank you also goes to Christian Gie{\ss}en for his collaboration in the preliminary versions of this paper.
Part of this work was funded by the EPSRC under grant agreement N. EP/M004252/1.


\bibliographystyle{plain}
  \bibliography{references}

\end{document}